\documentclass{article}

\usepackage{silence}
\usepackage{silence}
\usepackage{microtype}
\usepackage{graphicx}
\usepackage{subcaption}
\usepackage[table]{xcolor}
\usepackage{booktabs} 
\usepackage{amsmath}
\usepackage{amssymb}
\usepackage{mathtools}
\usepackage{amsthm}
\usepackage{array}
\usepackage{comment}
\usepackage{makecell}
\usepackage{calc}
\usepackage{pifont}
\usepackage[table]{xcolor}
\newcommand{\cmark}{\textcolor{green!60!black}{\ding{51}}} 
\newcommand{\xmark}{\textcolor{red}{\ding{55}}}
\usepackage[most]{tcolorbox}
\usepackage{kotex}
\usepackage{hyperref}

\usepackage{thmtools,thm-restate}
\usepackage{amsthm}

\usepackage[accepted]{icml2026}

\makeatletter

\let\algorithmicrequire\undefined
\let\algorithmicensure\undefined
\let\algorithmiccomment\undefined
\let\algorithmicend\undefined
\let\algorithmicif\undefined
\let\algorithmicthen\undefined
\let\algorithmicelse\undefined
\let\algorithmicfor\undefined
\let\algorithmicforall\undefined
\let\algorithmicdo\undefined
\let\algorithmicwhile\undefined
\let\algorithmicrepeat\undefined
\let\algorithmicuntil\undefined

\makeatother
\usepackage[dvipsnames]{xcolor}

\usepackage{algpseudocode}
\usepackage{amsmath,amsthm,amssymb,mathtools,bm}
\usepackage{mathcommand}
\usepackage{graphicx}
\usepackage{subcaption}
\usepackage{nicefrac}
\usepackage{comment}
\usepackage{booktabs}
\usepackage{multirow}
\usepackage{array}
\usepackage{tikz}
\usetikzlibrary{decorations.text}
\usepackage{pgfplots}
\usepackage{forest}
\usepackage{rotating}
\usepackage{bbding}
\usepackage{pgf}

\newtheoremstyle{funny}
  {}{}
  {\itshape}
  {}
  {\bfseries}
  {.}
  { }
  {%
   \thmname{#1}
   \thmnumber{ #2}
   \thmnote{ {\mdseries\iffunny(\fi#3\iffunny)\fi}}
   \global\funnytrue 
  }
\newif\iffunny

\theoremstyle{funny}

\newtheorem{proposition}{Proposition}
\theoremstyle{funny}

\algnewcommand\algorithmicswitch{\textbf{switch}}
\algnewcommand\algorithmiccase{\textbf{case}}
\algnewcommand\algorithmicassert{\texttt{assert}}
\algnewcommand\Assert[1]{\State \algorithmicassert(#1)}%
\algdef{SE}[SWITCH]{Switch}{EndSwitch}[1]{\algorithmicswitch\ #1\ \algorithmicdo}{\algorithmicend\ \algorithmicswitch}%
\algdef{SE}[CASE]{Case}{EndCase}[1]{\algorithmiccase\ #1}{\algorithmicend\ \algorithmiccase}%
\algtext*{EndSwitch}%
\algtext*{EndCase}%

\renewcommand{\cmark}{\textcolor{green}{\Checkmark}}
\renewcommand{\xmark}{\textcolor{red}{\XSolidBrush}}

\usepackage{colortbl}
\usepackage{tabularx}
\definecolor{bGreen}{HTML}{D3FFD3}
\definecolor{bBlue}{HTML}{D3D3FF}
\definecolor{bRed}{HTML}{FFD3D3}
\definecolor{bGray}{HTML}{D3D3D3}

\usepackage[capitalize,noabbrev]{cleveref}

\usepackage[textsize=tiny]{todonotes}

\icmltitlerunning{Triadic Dynamics Aware Diffusion Posterior Sampling for Inverse Problems}

\begin{document}

\twocolumn[
  \icmltitle{Triadic Dynamics Aware Diffusion Posterior Sampling for Inverse Problems: Optimizing Guidance and Stochasticity Schedules}



  \icmlsetsymbol{equal}{*}
  \begin{icmlauthorlist}
    \icmlauthor{Junseo Bang}{equal,ece}
    \icmlauthor{Dong Ju Mun}{equal,ece}
    \icmlauthor{Hoigi Seo}{ece}
    \icmlauthor{Seongmin Hong}{inmc}
    \icmlauthor{Se Young Chun}{ece,inmc,ipai}
  \end{icmlauthorlist}
  \icmlaffiliation{ece}{Dept. of Electrical and Computer Engineering}
  \icmlaffiliation{inmc}{INMC,}
    \icmlaffiliation{ipai}{IPAI \& AIIS, Seoul National University, Republic of Korea}
  \icmlcorrespondingauthor{Se Young Chun}{sychun@snu.ac.kr}
  \icmlkeywords{Diffusion Posterior Sampling, Inverse Problem, Optimal Guidance and Schedule}

  \vskip 0.3in
]



\printAffiliationsAndNotice{\icmlEqualContribution}


\begin{abstract}
Generative posterior sampling using diffusion models has emerged as a dominant paradigm for solving inverse problems in imaging, which usually consists of three main components: data consistency (DC) guidance, classifier-free guidance (CFG) and stochasticity. While prior arts have focused on how to develop each or all components, less attention has given to how to schedule them, leading to heuristically fixed or partially adjusted suboptimal schedules. In this work, we argue that the interactions among all three components in terms of scheduling are crucial for significantly improved performance in solving inverse problems in imaging. Our analysis shows that aggressive CFG early in sampling conflict with DC guidance, while stochasticity brings the trajectory back to higher-probability regions. Based on these findings, we propose Triadic Dynamics Aware Posterior Sampling (TriPS), which reformulates posterior sampling as a time-varying control problem and optimizes schedules following a triadic trend of decreasing DC and stochasticity scales alongside increasing CFG scale. TriPS achieves this through two strategies: template-based search over functional priors for reliable baseline schedules, and Group Relative Policy Optimization (GRPO)-based reinforcement learning for more flexible temporal curves. Experiments demonstrate TriPS outperforms state-of-the-art baselines in data fidelity and perceptual realism.

\end{abstract}

\section{Introduction}
\label{sec: intro}
Inverse problems aim to recover an unknown signal $\mathbf{x}_0 \in \mathbb{R}^d$ from noisy measurements $\bm{y} = \mathcal{A}\bm{x}_0 + \mathbf{n}$, where $\mathcal{A}: \mathbb{R}^d \rightarrow \mathbb{R}^m$ is a known forward operator and $\bm{n} \in \mathbb{R}^m \sim \mathcal{N}(0, \sigma_{\bm{n}}^2 I_m)$ represents measurement noise. While classical methods often rely on hand-crafted priors~\cite{rudin1992nonlinear, donoho1995noising, beck2009fast}, the paradigm has shifted toward leveraging expressive generative priors via diffusion models~\cite{ho2020denoising, song2020score} and flow matching~\cite{lipman2023flow, esser2024scaling}. 
Since the reverse diffusion process involves the score function of the prior, we can sample from the posterior by estimating the score of the posterior distribution~\cite{hyvarinen2005estimation, song2020score}:
\begin{equation}
\underbrace{\nabla_{\bm{x}_t} \log p(\bm{x}_t|\bm{y})}_{\text{Posterior score}} = \underbrace{\nabla_{\bm{x}_t} \log p(\bm{y}|\bm{x}_t)}_{\text{Log-likelihood score}} + \underbrace{\nabla_{\bm{x}_t} \log p(\bm{x}_t)}_{\text{Prior score}},
\end{equation}
where $t \in [0, T]$ denotes the timestep.

Existing methods focuses on explicitly approximating the log-likelihood score term through three primary categories: (i) projection-based methods project the intermediate state $\bm{x}_t$ or $\hat{\bm{x}}_{0|t} \coloneqq\mathbb{E}[\bm{x}_0|\bm{x}_t]$ onto the measurement subspace $\{\bm{x}|\mathcal{A}\bm{x}=\bm{y}\}$ via singular value decomposition or range-null space decomposition~\cite{snips, ddrm, ddnm}; (ii) gradient-based methods enforce consistency by taking a single gradient step on $\bm{x}_t$ to align the sample with the measurements at each timestep~\cite{dps, pigdm, blinddps, moment_matching, psld}; and (iii) optimization-based methods solve local or global optimization problems using proximal updates to approximate the posterior mean toward the measurement manifold~\cite{diffpir, dds, flower, dilack}. 

Complementing these likelihood-centric approaches, the emergence of large-scale text conditional generative models, such as Stable Diffusion~\cite{rombach2022high} and FLUX~\cite{esser2024scaling}, provides powerful implicit priors for solving inverse problems in a zero-shot manner. Since inverse problems are inherently ill-posed (\textit{e.g.}, $m < d$ in a linear $\mathcal{A}$) and measurements provide insufficient information, recent methods incorporate text prompts to steer the sampling process onto the semantic manifold aligned with the text conditioning, thereby regularizing the solution space~\cite{p2l, treg}

At the core of solving the reverse diffusion process within these diverse methods lie three fundamental components for posterior sampling:
\textbf{(1) Data consistency (DC) guidance}: Regularizing $\mathcal{A}\bm{x}_0$ to align with the measurements $\bm{y}$, with the strength controlled by the DC guidance scale, as in Eq.~\eqref{eq:dc_update}. \textbf{(2) Classifier-free guidance (CFG)}: Extrapolating the score estimate toward the conditioning text prompt with the guidance strength controlled by a CFG scale, as in Eq.~\eqref{eq:2.1}. \textbf{(3) Stochasticity}: Controlling the additive noise during the reverse diffusion process, with the noise strength governed by the stochasticity scale, as in Eq.~\eqref{eq:2.4}.

Although a \textit{time-varying} coordination of these components offers greater degrees of freedom to steer the sampling trajectory and optimize restoration quality, existing works typically treat them in a \textit{time-independent} manner. For instance, Flowchef~\cite{flowchef} employs constant values for the DC guidance and CFG scales without stochasiticity (\textit{i.e.}, $\lambda(t) = \lambda, \beta(t) = \beta$, $\eta(t)=0$). Conversely, some methods adjust only the DC guidance scale in a time-varying fashion to prevent over-saturation artifacts toward the end of sampling~\cite{pigdm, fast_samplers}.
While other methods have attempted to jointly schedule the DC and stochasticity scales to balance data consistency and perceptual quality~\cite{resample, ddpg, flowdps, flowlps, daps}, the interplay between these two components remains insufficiently characterized. Furthermore, time-varying designs for the CFG scale have yet to be explored in the context of inverse problems, despite their reported efficacy in enhancing performance and diversity within generation tasks~\cite{muse, wang2024analysis, limited_interval_cfg}.

To address this, we propose \textbf{TriPS} (Triadic Dynamics Aware Posterior Sampling), which reformulates posterior sampling as an optimization problem of time-varying schedules. TriPS characterizes DC guidance, CFG, and stochasticity as a coupled triadic system where early stage CFG conflicts with DC guidance, which hinders reducing DC. 
Furthermore, we empirically observe that calibrated stochasticity acts as a regularizer, counteracting the tendency of DC guidance and CFG to drive sampling trajectories away from higher-probability regions.
Building upon these insight, we introduce a triadic scheduling, characterized by monotonically decreasing DC and stochasticity scales while increasing CFG scale, and optimize through two complementary paradigms: (i) Template-based schedule search leverages functional priors to efficiently identify effective schedule by constraining the search space within a set of predefined templates. (ii) Group Relative Policy Optimization (GRPO)-based schedule optimization employs a reinforcement learning framework to capture complex temporal curves that transcend fixed functional forms, allowing it to effectively navigate the perception-distortion trade-off through a hybrid reward that jointly optimizes perceptual (\textit{e.g.}, LPIPS) and distortion metrics (\textit{e.g.}, PSNR). 
Our key contributions are as follow: 

\begin{itemize} \setlength\itemsep{-0.2em} \item \textbf{Analysis of Triadic Coupling Dynamics}: 
We identify the inherent conflict between DC guidance and CFG and reveal stochasticity as a regularizer that aligns sampling trajectories toward higher-probability regions.
\item \textbf{Triadic Schedule Optimization Framework}: We propose TriPS that finds effective time-varying schedules of DC, CFG, and stochasticity scales through the template-based search and GRPO-based optimization.
\end{itemize}

\section{Preliminaries}
\label{sec:preliminaries}

\subsection{Flow-based Posterior Sampling}
\label{subsec:flow_posterior}

We describe the general formulation of flow-based posterior sampling for solving inverse problems $\bm{y} = \mathcal{A}\bm{x}_0 + \mathbf{n}$. The sampling process at timestep $t$ is governed by a velocity field derived from a pretrained flow matching model, modified by DC guidance and CFG and stochastic injections.

First, the velocity field is adjusted by the CFG scale $\lambda(t)$ to incorporate conditional information $c$:
\begin{equation}\label{eq:2.1}
    v_t(\bm{x}_t) = v_\theta(\bm{x}_t, c_\emptyset) + \lambda(t)\big(v_\theta(\bm{x}_t, c) - v_\theta(\bm{x}_t, c_\emptyset)\big).
\end{equation}
Based on the velocity $v_t(\bm{x}_t)$, the clean data prediction $\hat{\bm{x}}_{0|t}$ and the corresponding noise prediction $\hat{\bm{x}}_{1|t}$ are estimated via the flow-based Tweedie formula~\cite{flowdps}: 
\begin{equation}
\label{eq:tweedie_update}
    \hat{\bm{x}}_{0|t} = \bm{x}_t - \sigma_t v_t(\bm{x}_t), \quad \hat{\bm{x}}_{1|t} = \bm{x}_t + (1-\sigma_t)v_t(\bm{x}_t),
\end{equation}
where $\sigma_t$ denotes the noise schedule. To enforce consistency with the measurement $\bm{y}$, a DC update is applied to the estimated clean data using a DC guidance scale $\beta(t)$:
\begin{equation}
\label{eq:dc_update}
    \tilde{\bm{x}}_{0|t}(\bm{y}) = \hat{\bm{x}}_{0|t} - \beta(t) \nabla_{\hat{\bm{x}}_{0|t}} \mathcal{L}(\mathcal{A}\hat{\bm{x}}_{0|t}, \bm{y}).
\end{equation}
Here, $\mathcal{L}$ denote the DC loss function (\textit{e.g.}, $\ell_2$ distance) that quantifies the discrepancy between the estimated forward projection and the measurement.
Concurrently, to inject stochasticity, the noise component $\hat{\bm{x}}_{1|t}$ is perturbed with a stochasticity scale $\eta(t)$ and Gaussian noise $\bm{\epsilon} \sim \mathcal{N}(0, I_d)$:
\begin{equation}\label{eq:2.4}
    \tilde{\bm{x}}_{1|t} = \sqrt{1-\eta^2(t)} \hat{\bm{x}}_{1|t} + \eta(t)\bm{\epsilon}.
\end{equation}
The state for the next timestep $\bm{x}_{t+\Delta t}$ is computed via an Euler method, combining the measurement aligned estimate $\tilde{\bm{x}}_{0|t}(\bm{y})$ with the stochastically modulated noise term $\tilde{\bm{x}}_{1|t}$:
\begin{equation}
    \bm{x}_{t+\Delta t} = (1-\sigma_{t+\Delta t})\tilde{\bm{x}}_{0|t}(\bm{y}) + \sigma_{t+\Delta t} \tilde{\bm{x}}_{1|t}.
\end{equation}
We also define $\bm{x}_{t+\Delta t}^{\text{det}}=(1-\sigma_{t+\Delta t})\tilde{\bm{x}}_{0|t}(\bm{y}) + \sigma_{t+\Delta t} \hat{\bm{x}}_{1|t}$ by setting $\eta(t)=0$ in the above procedure.

To explicitly analyze the contribution of each force driving the sampling trajectory, we define the guidance components. The total drift $b_{\text{det}}(\bm{x}_t)$ is defined as the finite difference velocity of the deterministic update:
\begin{equation}
    b_{\text{det}}(\bm{x}_t) = \frac{\bm{x}_{t+\Delta t}^{\text{det}} - \bm{x}_t}{\Delta t}.
\end{equation}
We decompose $b_{\text{det}}$ into $b_{\text{prior}}$, $b_{\text{cfg}}$, and $b_{\text{dc}}$:
\begin{equation}
\label{eq:drfits_definition}
\scalebox{0.92}{$\begin{aligned}
    b_{\text{prior}}(\bm{x}_t) \; &= v_\theta(\bm{x}_t; \emptyset), \\
    b_{\text{cfg}}(\bm{x}_t; c) &= \lambda(t) \big( v_\theta(\bm{x}_t; c) - v_\theta(\bm{x}_t; \emptyset) \big) = \lambda(t) \tilde{b}_{\text{cfg}}(\bm{x}_t; c), \\
    b_{\text{dc}}(\bm{x}_t;\bm{y}) &= b_{\text{det}}(\bm{x}_t) - b_{\text{prior}}(\bm{x}_t) - b_{\text{cfg}}(\bm{x}_t) = \beta(t) \tilde b_{\mathrm{dc}}(\bm{x}_t;\bm{y}).
\end{aligned}$}
\end{equation}
Here, $b_{\text{dc}}(\bm{x}_t;\bm{y})$ is the effective correction applied by the measurement gradients. The terms $\tilde b_{\mathrm{cfg}}$ and $\tilde b_{\mathrm{dc}}$ denote the unit scale guidance for CFG and DC guidance, respectively.

\subsection{TriPS Backbone Sampler}
\label{subsec:trips_backbone}
Based on the flow-based posterior sampling framework, we present the TriPS sampler. The primary distinction of this approach is that the DC guidance scale $\beta(t)$, CFG scale $\lambda(t)$, and stochasticity scale $\eta(t)$ are formulated as dynamic, time-varying functions rather than static constants. This formulation allows the sampler to adaptively regulate the sampling dynamics across different noise levels. For the DC loss function $\mathcal{L}(\cdot, \bm{y})$ in Eq.~\eqref{eq:dc_update}, we employ a hybrid update scheme that interpolates between Back-Projection and Least-Square objectives for robust reconstruction~\cite{ddpg}. The complete procedure is detailed in Appendix~\ref{app:sec_algo}.

\begin{figure}[t]
  \centering
  \includegraphics[width=0.50\textwidth]{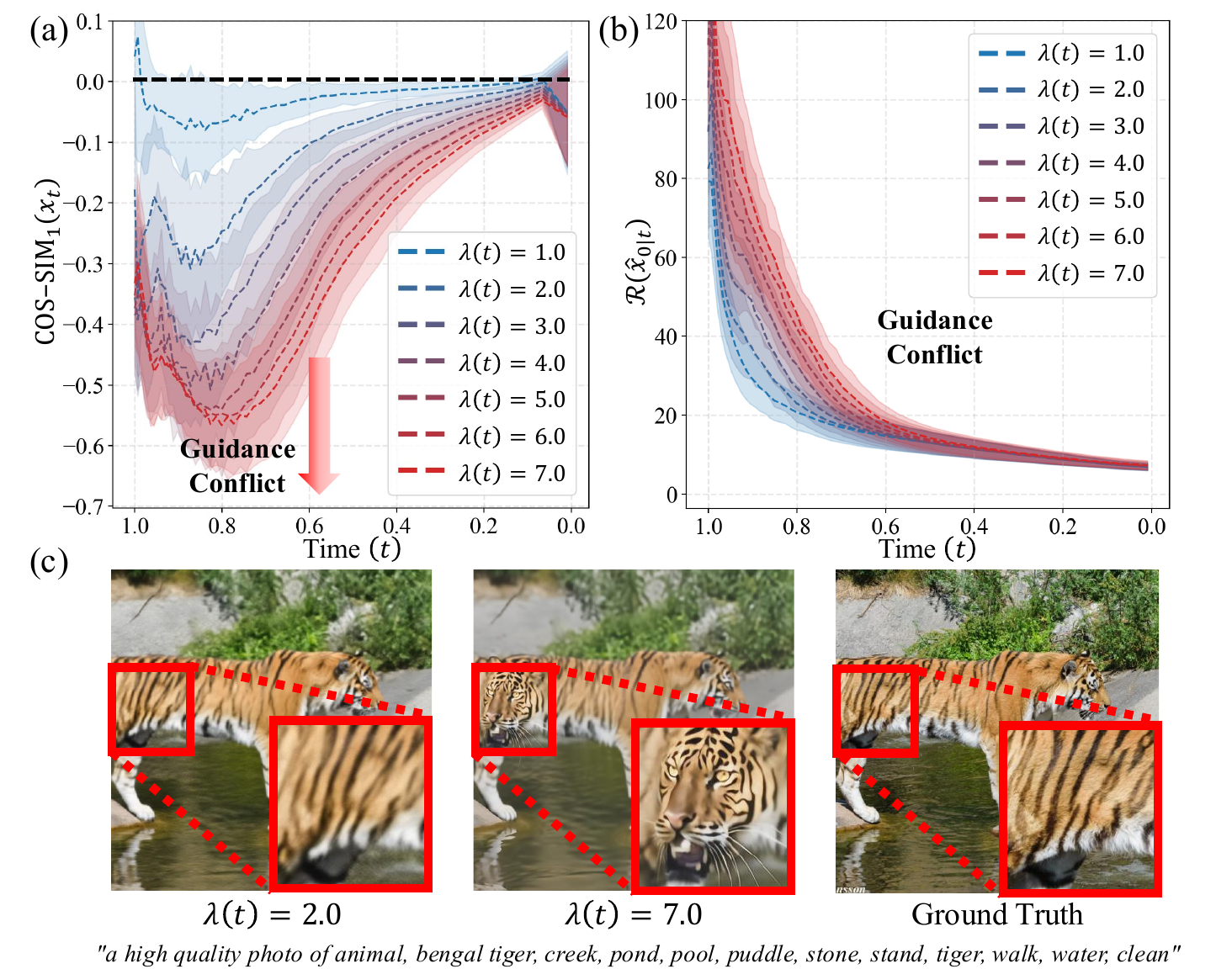}
  \caption{Early stage guidance conflict on super-resolution $\times8$. (a) Cosine similarity $\text{COS-SIM}_{1}(\bm{x}_t)$ between the unit-scale DC guidance $\tilde{b}_{\mathrm{dc}}$ and the unit-scale CFG $\tilde{b}_{\mathrm{cfg}}$ across sampling timesteps, shown for varying CFG scales $\lambda(t)$. As $\lambda(t)$ increases, $\text{COS-SIM}_{1}(\bm{x}_t)$ at early timesteps becomes more negative. (b) Decay of the squared residual norm $\mathcal{R}(\hat{\bm{x}}_{0|t})$, defined in Proposition~\ref{prop:cfg_slows_dc}, for varying CFG scales $\lambda(t)$. (c) Visual comparison on the DIV2K dataset. High CFG scales lead to semantic hallucinations (\textit{e.g.}, distorted tiger patterns) that deviate from the measurement $\bm{y}$.
  }
  \vspace{-1em}
  \label{fig:analysis_dc_cfg}
\end{figure}

\section{Analysis: Triadic Coupling Dynamics}
\label{sec:analysis}
This section analyzes the triadic coupling dynamics in posterior sampling, governed by DC guidance, CFG, and stochasticity. 
We formalize the early-stage conflict of DC guidance, CFG and demonstrate how stochasticity regularizes sampling trajectories towards higher-probability regions.

\subsection{Early-Stage Guidance Conflict}
\label{subsec:dc_cfg_conflict}
The decomposition in $b_{\text{det}}(\bm{x}_t)$ in Eq.~\eqref{eq:drfits_definition} incorporates two potentially competing components: $b_{\text{dc}}(\bm{x}_t;\bm{y})$, which enforces DC to align with the measurements $\bm{y}$, and $b_{\text{cfg}}(\bm{x}_t; c)$, which steers the sampling trajectory toward conditioning text prompt. Since these guidance originate from distinct objectives, their update directions are misaligned, leading to a conflict of DC guidance and CFG. To formalize this, we analyze the behavior of the residual norm $\mathcal{R}(\hat{\bm{x}}_{0|t}) \triangleq \|\bm{y} - \mathcal{A}\hat{\bm{x}}_{0|t}\|_2^2$ with respect to the CFG scale $\lambda(t)$. This derivative is validated through super-resolution $\times8$, averaged over 100 samples on the DIV2K dataset~\cite{div2k}, where we utilize a fixed set of prompts consistent with our baseline.

\begin{proposition}[(first-order derivative of the expected residual norm descent to CFG scale)]
\label{prop:cfg_slows_dc}
Let $\tilde b_{\mathrm{dc}}$, $\tilde{b}_{\mathrm{cfg}}$ be the unit-scale DC guidance, CFG defined in Eq.~\eqref{eq:drfits_definition} and $\mathcal{R}(\hat{\bm{x}}_{0|t}) \triangleq  \|\bm{y} - \mathcal{A}\hat{\bm{x}}_{0|t}\|_2^2$ be the residual norm.

If $\mathcal{R}\in C^2$, the first-order derivative of the next-step expected residual norm conditioned on $\bm{x}_t$ with respect to the CFG scale $\lambda(t)$ is:
\begin{equation}
\label{eq:dlambda}
\begin{aligned}
&\frac{\partial}{\partial \lambda(t)}\,
\mathbb{E}\!\left[\mathcal{R}(\hat{\bm{x}}_{0|t+\Delta t})) \,\middle|\, \bm{x}_t\right]
= \\
&-\Delta t\,\Big\langle \tilde b_{\mathrm{dc}}(\bm{x}_t;y),\,\tilde b_{\mathrm{cfg}}(\bm{x}_t;c)\Big\rangle
+o(\Delta t),
\end{aligned}
\end{equation}
where $\langle \cdot, \cdot \rangle $ denotes the standard Euclidean inner product.
\end{proposition}
(Proof in Appendix~\ref{app:subsec_proof_cfg_slows_dc}). Proposition~\ref{prop:cfg_slows_dc} implies that whenever the CFG opposes the DC guidance direction (\textit{i.e.}, $\langle \tilde b_{\mathrm{dc}}(\bm{x}_t;y), \tilde b_{\mathrm{cfg}}(\bm{x}_t;c)\rangle \le 0$), increasing the CFG scale $\lambda(t)$ impairs the minimization of the expected residual norm. In such cases, a large $\lambda(t)$ slows the descent of the $\mathcal{R}(\hat{\bm{x}}_{0|t})$. To quantify this misalignment, we define the cosine similarity $\text{COS-SIM}_{1}(\bm{x}_t)$ between the unit-scale DC guidance $\tilde b_{\mathrm{dc}}$ and the unit-scale CFG $\tilde b_{\mathrm{cfg}}$:
\begin{equation}
\label{eq:cos_sim_def}
\text{COS-SIM}_{1}(\bm{x}_t) \triangleq \frac{\Big\langle \tilde b_{\mathrm{dc}}(\bm{x}_t;y), \, \tilde b_{\mathrm{cfg}}(\bm{x}_t;c) \Big\rangle}{\big\| \tilde b_{\mathrm{dc}}(\bm{x}_t;y) \big\|_2 \, \big\| \tilde b_{\mathrm{cfg}}(\bm{x}_t;c) \big\|_2},
\end{equation}
Empirical results in Fig.~\ref{fig:analysis_dc_cfg} support this theoretical analysis. Fig.~\ref{fig:analysis_dc_cfg}(a) shows that $\text{COS-SIM}_{1}(\bm{x}_t)$ is predominantly negative at early timesteps (i.e., $t \simeq 1$), indicating that the CFG initially counteracts the DC guidance. As $t \to 0$, $\text{COS-SIM}_{1}(\bm{x}_t)$ approaches zero, suggesting that the directional misalignment diminishes. Consistent with Eq.~\eqref{eq:dlambda}, this directional misalignment causes larger CFG scales $\lambda(t)$ to hinder squared residual norm minimization, as evidenced by the slower decay of $\mathcal{R}(\hat{\bm{x}}_{0|t})$ in Fig.~\ref{fig:analysis_dc_cfg}(b). These observations are qualitatively supported by Fig.~\ref{fig:analysis_dc_cfg}(c), where high CFG scales induce semantic hallucinations that violate DC, whereas lower scales maintain DC.

\begin{figure}[t]
  \centering
  \includegraphics[width=1\linewidth]{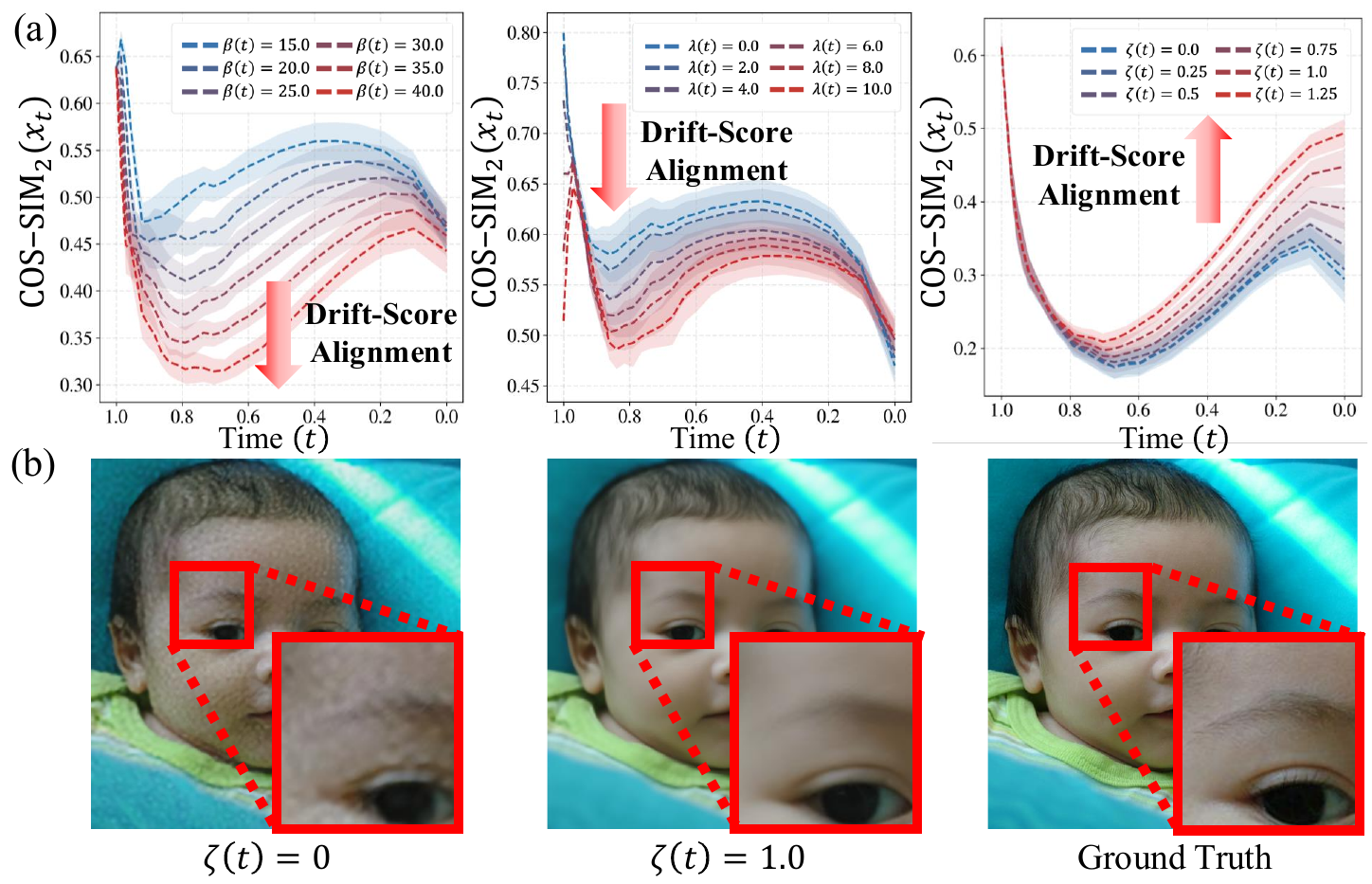}
  \caption{
  Stochasticity as a regularizer for DC guidance and CFG on super-resolution $\times 8$.
  Note that the stochasticity scale is scheduled as $\eta(t) = \zeta(t)\sqrt{1-\sigma_{t+\Delta t}}$ for this analysis.
  (a) Cosine similarity $\text{COS-SIM}_{2}(\bm{x}_t)$ between the total drift $b_{\text{det}}(\bm{x}_t)$ and the score function $\nabla_{\bm{x}_t}\log p_t(\bm{x}_t)$ across sampling timesteps, shown for varying DC guidance scales $\beta(t)$ (left), CFG scales $\lambda(t)$ (middle), and stochasticity scales $\eta(t)$ (right).
  Increasing scales of DC guidance or CFG leads to lower $\text{COS-SIM}_{2}(\bm{x}_t)$ while increasing the stochasticity scale tends to produce higher $\text{COS-SIM}_{2}(\bm{x}_t)$. 
  (b) Visual comparisons on the FFHQ dataset. The proper injection of stochasticity mitigates over-saturation artifacts, thereby restoring perceptual fidelity and improving overall sample quality.
  }
  \vspace{-1em}
  \label{fig:analysis_stochasticity}
\end{figure}


\subsection{Stochasticity as a Regularizer for DC and CFG}
\label{subsec:analysis_sto_regularize}
We examine how stochasticity actively regularizes sampling trajectories to higher-probability regions, specifically counteracting the off-manifold phenomenon, where DC guidance and CFG steer the sampling trajectory away from higher-probability regions.
To quantify how closely the total drift $b_{\text{det}}(\bm{x}_t)$ aligns with the unconditional score function at each timestep $t$,
we define the following cosine similarity:
\begin{equation}
\label{eq:cos_sim_2_def}
\text{COS-SIM}_{2}(\bm{x}_t) \triangleq \frac{\Big\langle b_{\text{det}}(\bm{x}_t), \, \nabla_{\bm{x}_t}\log p_t(\bm{x}_t) \Big\rangle}{\big\| b_{\text{det}}(\bm{x}_t) \big\|_2 \, \big\| \nabla_{\bm{x}_t}\log p_t(\bm{x}_t) \big\|_2},
\end{equation}
where $\nabla_{\bm{x}_t}\log p_t(\bm{x}_t)$ is the unconditional score function pointing toward higher-probability regions of the data distribution. The explicit formulation of the score function within the flow-based framework is detailed in Appendix~\ref{app:subsec_score_from_velocity}.
Larger $\text{COS-SIM}_{2}(\bm{x}_t)$ values indicate that the total drift aligns more strongly with the score direction, steering the sample back to the data manifold, while lower values reflect a direction pointing towards lower-probability regions.

Empirical results in Fig.~\ref{fig:analysis_stochasticity}(a) show that increasing the DC guidance or CFG scale reduces $\text{COS-SIM}_{2}(\bm{x}_t)$, indicating that stronger guidance drives sampling away from higher-probability trajectories. In contrast, increasing the stochasticity scale consistently improves $\text{COS-SIM}_{2}(\bm{x}_t)$, suggesting that appropriately scaled noise enhances alignment with the data distribution. Moreover, Fig.~\ref{fig:analysis_stochasticity}(a) (right) shows that higher stochasticity applied in early sampling stages yields more robust alignment with the score function at later timesteps. This regularizing effect is qualitatively reflected in Fig.~\ref{fig:analysis_stochasticity}(b), where suitable stochasticity suppresses artifacts and restores perceptual fidelity.

To further validate these effects from a marginal distributional perspective, we evaluate the Kernel Inception Distance (KID) between generated samples $\bm{x}_0$ and ground-truth images (see Appendix~\ref{app:subsec_distribution_anal_KID}). Increasing the scale of DC guidance or CFG consistently enlarges the distributional discrepancy, whereas introducing stochasticity reduces this gap by aligning the marginal distribution of generated samples more closely with that of the ground truth. This indicates that stochasticity acts as a regularizer, guiding sampling toward higher-probability regions of the data distribution and counteracting the tendency of DC guidance and CFG to drive samples toward lower-probability regions.

\begin{figure*}[t]
  \centering
  \includegraphics[width=1\textwidth]{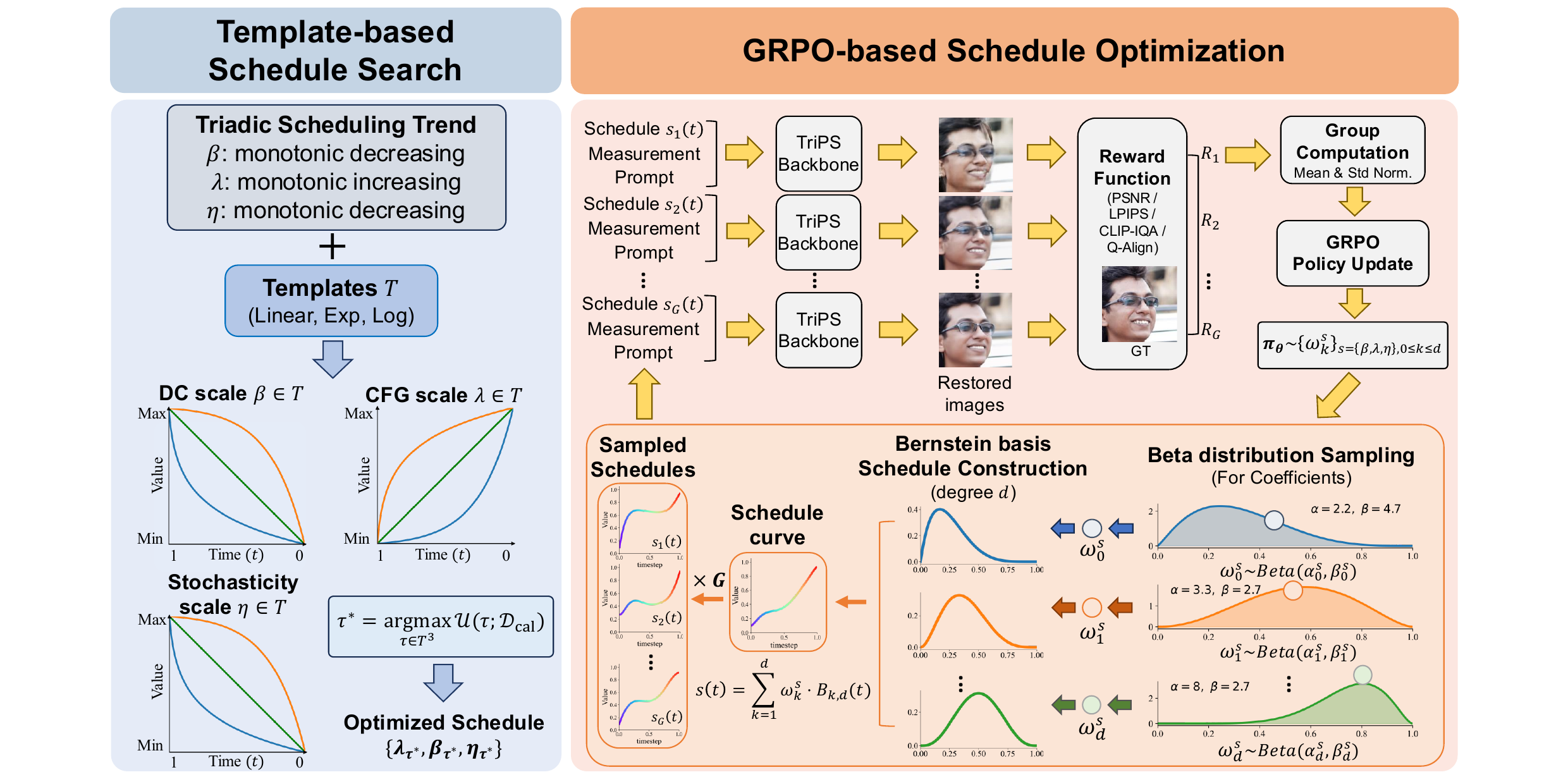}
  \caption{
  Overview of the triadic schedule optimization framework. 
  (Left) Template-based schedule search explores a discrete search space defined by compact templates (Linear, Exp, Log) that satisfy the triadic scheduling trend ($\beta(t)\!\downarrow, \lambda(t)\!\uparrow, \eta(t)\!\downarrow$). 
  (Right) GRPO-based optimization facilitates continuous schedule discovery beyond fixed functional forms. A policy $\pi_\theta$ samples coefficients $w_k^{(s)}$ from learnable Beta distributions to parameterize continuous schedules $s(t)$ via Bernstein polynomials $B_{k,d}(t)$. This formulation strictly constrains $s(t)$ within the predefined range, while $\pi_\theta$ is updated via rewards derived from the restored images. 
  }
  \vspace{-1em}
  \label{fig:method_overview}
\end{figure*}

\subsection{Triadic Scheduling Trend}
\label{subsec:tri_sched_trend}
We adopt a structured scheduling strategy for the DC guidance scale $\beta(t)$, CFG scale $\lambda(t)$, and stochasticity scale $\eta(t)$ that aligns with the coarse-to-fine nature of diffusion sampling. 
In the early high-noise regime, we deliberately set a high DC guidance scale $\beta(t)$ to strongly enforce data consistency with the measurements $\bm{y}$ following a practice in prior works~\cite{ddpg, flowdps}.
This choice forms the foundation of our design, enabling subsequent scheduling decisions to arise naturally from our analyses.
We keep a low $\lambda(t)$ to suppress the early-stage guidance conflict (Sec.~\ref{subsec:dc_cfg_conflict}).
We also set $\eta(t)$ to be a high value because stochasticity in early timestep regularizes the off-manifold phenomenon by promoting alignment of the total drift with the score function in later timestep (Sec.~\ref{subsec:analysis_sto_regularize}).

As the process transitions toward the low-noise regime, focus shifts to detail refinement. With global structure established and the residual norm (Proposition~\ref{prop:cfg_slows_dc}) reduced, the guidance conflict diminishes, allowing increase of $\lambda(t)$ to leverage the mode-seeking properties and sharpen semantics for enhanced fidelity~\cite{ho2022classifier, wang2024analysis}.
Concurrently, we reduce $\beta(t)$ to mitigate the guidance conflict and avoid the over-enforcing of data consistency, which can inject measurement noise into the output~\cite{ddnm}. Finally, we anneal the stochasticity scale $\eta(t)$ to a lower value, as late stage stochasticity introduces sampling errors~\cite{restart} and reducing $\eta(t)$ achieves more precise approximation of the target distribution~\cite{ecva}. This triadic scheduling trend is now summarized as: \textit{Monotonic decrease in $\beta(t)$, increase in $\lambda(t)$ and decrease in $\eta(t)$}.

\section{Method: Triadic Schedule Optimization}
\label{sec: method}
To implement the triadic scheduling trends from Sec.~\ref{subsec:tri_sched_trend}, we propose a framework comprising two complementary paradigms. 
First, template-based schedule search exploits a discrete family of functional forms to identify robust schedule curves.
Second, GRPO-based schedule optimization enables fine-grained schedule adaptation to capture complex temporal curves that transcend the fixed functional templates.
Both approaches utilize the backbone solver described in Sec.~\ref{subsec:trips_backbone} as the underlying sampling mechanism.

\begin{table*}[t!]
    \centering
    \renewcommand{\arraystretch}{0.9} 
    \setlength{\tabcolsep}{0.2em} 
    \caption{Quantitative comparison on linear inverse problems. Results are averaged over 1,000 FFHQ and 800 DIV2K samples using 28 NFEs with Gaussian noise ($\sigma_n = 0.03$). \textbf{Bold} and \underline{underline} denote the best and second-best performance, respectively.}
    \resizebox{0.99\textwidth}{!}{
    \begin{tabular}{l|ccccc|ccccc|ccccc|ccccc}
    \toprule
    \multicolumn{21}{c}{\textbf{Flow Matching Model (SD3.5-M)}}\\
    \midrule
    & \multicolumn{5}{c|}{Super-Resolution $\times$8}
    & \multicolumn{5}{c|}{Super-Resolution $\times$12}
    & \multicolumn{5}{c|}{Motion Deblurring}
    & \multicolumn{5}{c}{Gaussian Deblurring} \\
    \midrule
    Method
    & PSNR$\uparrow$ & SSIM$\uparrow$ & FID$\downarrow$ & LPIPS$\downarrow$ & MUSIQ$\uparrow$
    & PSNR$\uparrow$ & SSIM$\uparrow$ & FID$\downarrow$ & LPIPS$\downarrow$ & MUSIQ$\uparrow$
    & PSNR$\uparrow$ & SSIM$\uparrow$ & FID$\downarrow$ & LPIPS$\downarrow$ & MUSIQ$\uparrow$
    & PSNR$\uparrow$ & SSIM$\uparrow$ & FID$\downarrow$ & LPIPS$\downarrow$ & MUSIQ$\uparrow$ \\
    \midrule
    \multicolumn{21}{c}{\textbf{FFHQ (768 $\times$ 768)}}\\
    \midrule
    ReSample & 24.65 & 0.708 & 98.07 & 0.196 & 32.06 & 
    22.96 & 0.632 & 172.19 & 0.357 & 21.09 & 
    25.54 & 0.747 & 96.62 & 0.216 & 31.65 & 
    26.37 & 0.746 & 97.47 & 0.160 & 31.50 \\
    FlowChef & 27.53 & 0.759 & 57.24 & 0.147 & 49.05 & 
    26.38 & 0.731 & 110.15 & 0.209 & 38.70 & 
    24.88 & 0.716 & 63.48 & 0.237 & 39.22 & 
    27.30 & 0.754 & 46.40 & 0.152 & 44.26 \\
    FlowDPS & 27.92 & \underline{0.772} & \underline{23.80} & 0.120 & \underline{54.36} & 
    26.84 & 0.745 & 30.54 & \underline{0.156} & 48.33 & 
    25.15 & 0.721 & 43.18 & 0.222 & 48.51 & 
    26.02 & 0.731 & 45.00 & 0.204 & 47.11 \\
    FLAIR & \underline{28.88} & 0.768 & 55.51 & 0.123 & 52.26 & 
    27.25 & 0.752 & \underline{29.31} & 0.158 & 46.63 & 
    28.80 & 0.695 & 21.57 & 0.095 & 52.48 & 
    28.60 & 0.729 & \textbf{18.41} & 0.090 & \underline{54.71} \\
    \midrule
    \rowcolor{red!10}
    $\text{TriPS}_{\text{T}}$ (Ours) & \textbf{29.03} & \textbf{0.789} & 26.66 & \underline{0.113} & 53.65 & 
    \textbf{27.45} & \underline{0.754} & 35.13 & 0.161 & \underline{52.63} & 
    \textbf{31.20} &  \underline{0.809} & \underline{17.28} & \underline{0.060} & \underline{63.52} & 
    \textbf{29.95} & \textbf{0.804} & 25.02 & \underline{0.084} & 51.12 \\
    \rowcolor{green!10}
    $\text{TriPS}_{\text{G}}$ (Ours) & 28.55 & 0.762 & \textbf{22.18} & \textbf{0.107} & \textbf{61.69} & 
    \underline{27.38} & \textbf{0.762} & \textbf{28.22} & \textbf{0.154} & \textbf{53.92} & 
    \textbf{31.20} & \textbf{0.813} & \textbf{15.89} & \textbf{0.059} & \textbf{64.37} & 
    \underline{29.60} & \underline{0.782} & \underline{21.21} & \textbf{0.074} & \textbf{61.92} \\
    \midrule
    \multicolumn{21}{c}{\textbf{DIV2K (768 $\times$ 768)}}\\
    \midrule
    ReSample & 20.55 & 0.535 & 75.67 & 0.238 & 26.82 & 
    19.54 & 0.459 & 119.22 & 0.372 & 20.57 & 
    21.47 & 0.588 & 61.79 & 0.222 & 31.17 & 
    21.74 & 0.559 & 80.79 & 0.221 & 24.22 \\
    FlowChef & 22.08 & 0.561 & 47.47 & 0.213 & 41.70 & 
    20.88 & 0.508 & 84.28 & 0.297 & 35.94 & 
    19.62 & 0.482 & 74.01 & 0.366 & 41.24 & 
    21.57 & 0.532 & 52.84 & 0.251 & 39.52 \\
    FlowDPS & 22.14 & 0.545 & 35.18 & 0.175 & 39.87 & 
    21.07 & 0.469 & 48.36 & \textbf{0.251} & 33.47 & 
    19.88 & 0.473 & 52.23 & 0.322 & 40.01 & 
    20.46 & 0.473 & 58.56 & 0.307 & 36.80 \\
    FLAIR & \underline{22.90} & 0.592 & 41.23 & 0.167 & 42.30 & 
    21.12 & \underline{0.520} & \underline{42.16} & 0.256 & 46.23 & 
    23.90 & 0.614 & 22.17 & 0.129 & 52.24 & 
    22.70 & 0.561 & 32.26 & 0.157 & \underline{44.96} \\
    \midrule
    \rowcolor{red!10}
    $\text{TriPS}_{\text{T}}$ (Ours) & \textbf{23.05} & \textbf{0.607} & \underline{31.80} & \textbf{0.158} & \underline{45.24} & 
    \textbf{21.27} & 0.518 & 43.74 & \underline{0.255} & \underline{50.14} & 
    \textbf{26.29} & \textbf{0.728} & \underline{15.49} & \textbf{0.066} & \underline{59.16} & 
    \underline{23.94} & \textbf{0.646} & \underline{28.57} & \underline{0.123} & 44.47 \\
    \rowcolor{green!10}
    $\text{TriPS}_{\text{G}}$ (Ours) & 22.78 & \underline{0.594} & \textbf{27.84} & \underline{0.163} & \textbf{50.14} & 
    \underline{21.13} & \textbf{0.531} & \textbf{37.48} & 0.257 & \textbf{52.26} & 
    \underline{26.19} & \underline{0.715} & \textbf{14.94} & \textbf{0.066} & \textbf{59.89} & 
    \textbf{23.97} & \underline{0.644} & \textbf{26.88} & \textbf{0.121} & 
    \textbf{45.19} \\
    \bottomrule
    \end{tabular}
    }
    \vspace{-1em}
    \label{tab:quanti_flow}
\end{table*}

\subsection{$\text{TriPS}_{\text{T}}$: Template-based Schedule Search} \label{subsec:template_opt}
To instantiate the triadic scheduling trends identified in Sec.~\ref{subsec:tri_sched_trend}, we first perform a coarse optimization over a discrete template family $\mathcal{T} = \{\text{linear, exponential, logarithmic}\}$ as shown in Fig.~\ref{fig:method_overview}. These templates are selected to span a diverse range of temporal dynamics while strictly adhering to the established monotonicity constraints. Specifically, linear templates maintain a constant rate of change across all sampling timesteps. The exponential templates are designed to capture trajectories with an accelerated shift which is characterized by gradual adjustments in early timesteps followed by rapid transitions toward the end, whereas logarithmic templates model the inverse behavior. By constraining the search to these functional priors, we transform the high-dimensional optimization of per-timestep parameters into a low-dimensional selection problem, significantly reducing the degrees of freedom to ensure practical search efficiency.

Each candidate schedule is parameterized as a set $\tau \in \mathcal{T}^3$, with magnitudes bounded by $\lambda(t) \in [1.0, 6.0]$, $\eta(t) \in [0, 1]$, and $\beta(t) \in [\beta_{\min}^{\text{T}}, \beta_{\max}^{\text{T}}]$. The optimal template-based schedule $\tau^\star$ is identified via a systematic grid search maximizing a multi-objective utility function $\mathcal{U}$ on a small calibration set $\mathcal{D}_{\text{cal}}$, balancing fidelity (PSNR) and perceptual quality (LPIPS~\cite{lpips}):
\begin{equation}
\tau^\star \;=\; \arg\max_{\tau\in\mathcal{T}^3}\; \mathcal{U}(\tau;\mathcal{D}_{\text{cal}}).
\label{eq:template_search}
\end{equation}
The resulting schedule $\mathbf{S}_{\text{T}} = \{\lambda_{\tau^\star}, \beta_{\tau^\star}, \eta_{\tau^\star}\}$ provides a robust baseline that also serves as principled warm-start for the subsequent GRPO-based schedule optimization.

\subsection{$\text{TriPS}_{\text{G}}$: GRPO-based Schedule Optimization} \label{subsec:grpo_opt}
\paragraph{Bernstein-Beta schedule parameterization}
\label{par: schedule_parameterization}
To capture complex temporal dynamics that exceed the representational capacity of fixed templates, we introduce a data-driven schedule optimizaiton framework using Group Relative Policy Optimization (GRPO)~\cite{shao2024deepseekmath}, which estimates policy gradients via group-wise reward standardization without a value function. We parameterize each schedule component $s = \{\lambda, \beta, \eta\}$ as a continuous curve: 
\begin{equation}
\tilde{s}(t) = \sum_{k=0}^d w_k^{(s)} B_{k,d}(t),
\end{equation}
where $w_k^{(s)}$ denotes coefficients and $B_{k,d}(t) \;=\; \binom{d}{k}\, t^k (1-t)^{d-k}$ represents the Bernstein basis functions of degree $d$.
We exploit the convex hull property of Bernstein polynomials and the bounded support of Beta distributions to intrinsically confine the optimization within physically valid ranges and mitigate policy divergence during GRPO exploration. Specifically, we define a stochastic policy $\pi_\theta$ where the optimization variables $\theta = \{a_k^{(s)}, b_k^{(s)}\}_{s,k}$ parameterizes Beta distributions from which coefficients are sampled as:
\begin{equation}
w^{(s)}_k \sim \mathrm{Beta}\big(a^{(s)}_k,b^{(s)}_k\big), \quad s=\{\lambda,\beta,\eta\}
\end{equation}
Crucially, since the Beta samples are bounded in $(0,1)$ and the Bernstein basis forms a partition of unity ($\sum_k B_{k,d}(t)=1$), the resulting curve $\tilde{s}(t)$ represents a convex combination guaranteed to lie within $(0,1)$ for all $t$.
These normalized curves are mapped to physical scales via affine rescaling $s(t)= s_{\min} + (s_{\max}-s_{\min})\,\tilde{s}(t)$.
The boundaries $[s_{\min}, s_{\max}]$ for each parameter $s$ are fixed constants shared across all tasks and datasets which are detailed in the Appendix~\ref{app:subsec_implement_grpo}.
We denote the aggregated vector of all sampled coefficients as $\mathbf{w} = \{w_k^{(s)}\}_{s = \{\lambda,\beta,\eta\}, 0\le k \le d}$.

\paragraph{Hybrid IQA reward and group relative policy optimization.}
We drive optimization with a hybrid reward $R = w_{\text{dist}}R_{\text{dist}} + w_{\text{perc}}R_{\text{perc}}$ that combines distortion (PSNR) and perceptual metrics (LPIPS~\cite{lpips}, CLIP-IQA+~\cite{clipiqa}, Q-Align~\cite{qalign}). All metrics are unified to a monotonically increasing scale to ensure consistent optimization direction.
At each update step, we sample a group of $G$ coefficient vectors $\{\mathbf{w}_i\}_{i=1}^G$ from the current policy $\pi_{\theta_{\text{old}}}$, compute their advantages $\hat{A}_i$ via group-wise standardization, and updates the policy using the clipped surrogate objective:
\begin{equation}
\begin{aligned}
\max_{\theta}\;
\mathbb{E}_{i}\Big[
\min\!\big(r_i(\theta)\hat{A}_i,\,
\mathrm{clip}(r_i(\theta),1-\epsilon,1+\epsilon)\hat{A}_i\big) \\
-\beta_{\text{KL}}\,D_{\mathrm{KL}}\!\big(\pi_\theta\,\|\,\pi_{\text{ref}}\big)
\Big]
\end{aligned}
\raisetag{1.5ex}
\end{equation}
where $r_i(\theta)=\pi_\theta(\mathbf{w}_i)/\pi_{\theta_{\text{old}}}(\mathbf{w}_i)$ and $\pi_{\text{ref}}$ is the fixed policy reference whose parameters are initialized to approximate the schedule $\mathbf{S}_{\text{T}}$ obtained via template-based schedule search.
The overall optimization pipeline is illustrated in Fig.~\ref{fig:method_overview}.
The resulting schedule enables a controller of the perception-distortion trade-off by discovering complex temporal dynamics inaccessible to fixed functional forms.

\begin{figure*}[t]
  \centering
  \includegraphics[width=0.95\textwidth]{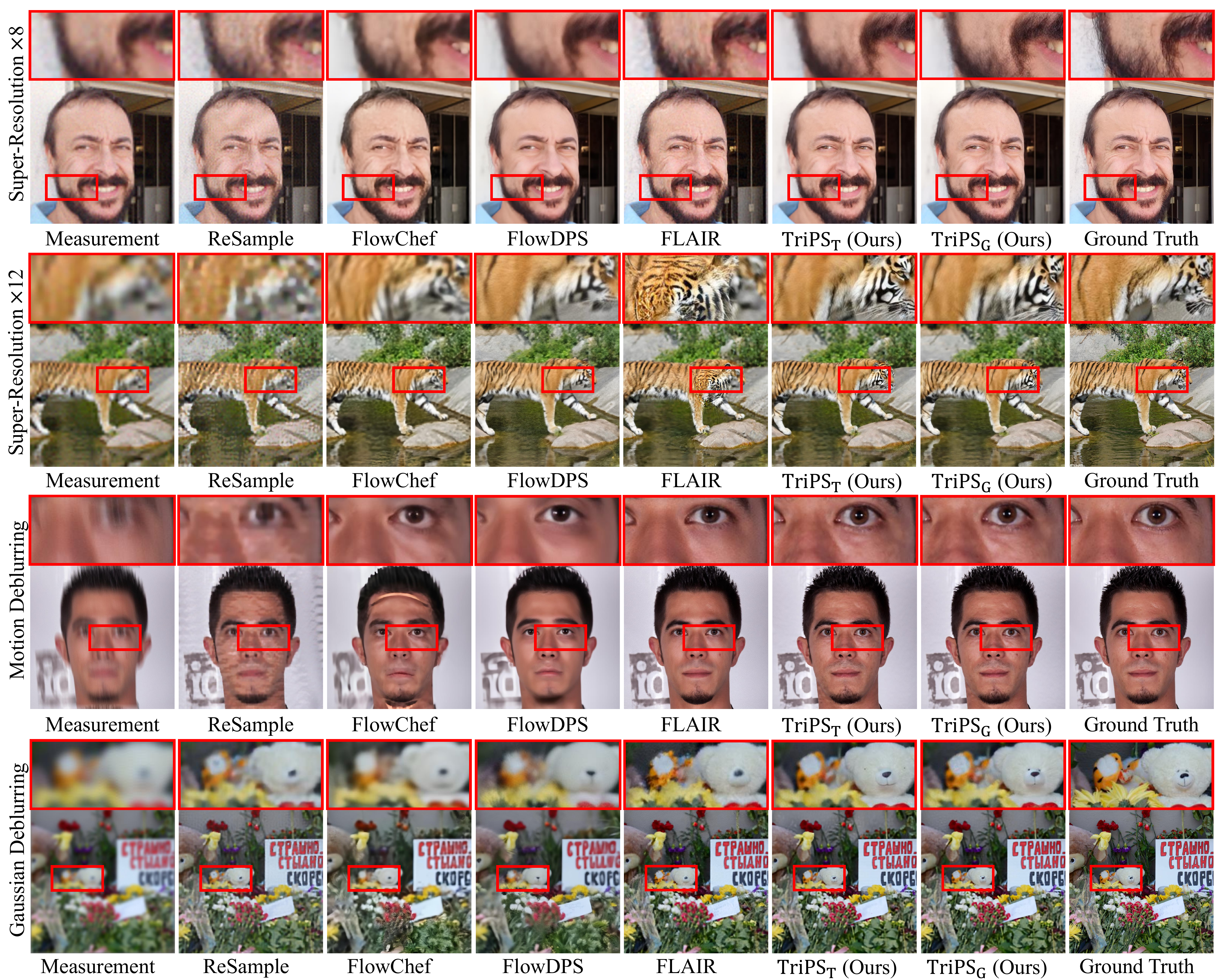}
  \caption{
  Qualitative comparison for FFHQ and DIV2K datasets on linear inverse problems. Best view in zoom. 
  }
  \label{fig:quali_flow}
\end{figure*}

\section{Experiments}
\label{sec: exp}

\subsection{Experimental Setup}
\paragraph{Datasets and metrics} \label{par:data_metric}
We evaluate our method on the high-resolution FFHQ~\cite{ffhq} ($1024^2$, 1k samples) and DIV2K~\cite{div2k} (2K, 800 samples) benchmarks. Quantitative assessment employs PSNR and SSIM~\cite{ssim} as distortion metrics, alongside LPIPS~\cite{lpips} and patch-based FID~\cite{fid} for perceptual realism. We additionally report MUSIQ~\cite{musiq}, a no-reference IQA metric, to effectively capture perception at high resolutions.

\paragraph{Baselines}
To evaluate our method, we benchmark against state-of-the-art baselines categorized by their generative backbone. For flow matching models (Stable Diffusion 3.5-Medium, $768^2$, NFE 28), we evaluate ReSample~\cite{resample}, FlowChef~\cite{flowchef}, FlowDPS~\cite{flowdps}, and FLAIR~\cite{flair}. Detailed hyperparameter configurations are provided in the Appendix~\ref{app:subsec_baseline_hyperpar}

\paragraph{Problem setting}
We evaluate our approach on multiple inverse problems under task-specific degradation settings. Within the flow matching framework, we test super-resolution 8$\times$, 12$\times$ using bicubic downsampling, motion and Gaussian deblurring with $61\times61$ kernels at intensity 0.5 and standard deviation 3.0, respectively. 
In all cases, a Gaussian noise with a standard deviation of 0.03 (1.5\% of the pixel range) is added. We use the fixed prompt "A high quality photo of a face" for FFHQ and image-specific text prompts generated by DAPE~\cite{seesr} for DIV2K.


\subsection{Experimental Results}
\label{subsec:main_exp_results}
\paragraph{Flow matching model}
As detailed in Table~\ref{tab:quanti_flow}, $\text{TriPS}_{\text{T}}$ consistently outperforms flow matching baselines. $\text{TriPS}_{\text{G}}$ further advances these results, significantly improving perceptual metrics while maintaining high measurement consistency. These quantitative gains are demonstrated in Fig.~\ref{fig:quali_flow}, where both variants effectively mitigate structural artifacts and noise present in baseline reconstructions thereby enhancing visual quality. Extended evaluations on diffusion model backbones are provided in the Appendix~\ref{app:subsec_trips_diffusion}.

\begin{table}[t!]
    \centering
    \renewcommand{\arraystretch}{0.9}
    \setlength{\tabcolsep}{1.25em}
    \caption{
    Quantitative comparison of the $\text{TriPS}_{\text{G}}$, optimized for SR $\times 8$, against baselines under different degradation settings: (top) Gaussian deblurring and (bottom) SR $\times 12$ on FFHQ 100 samples.
    }
    \label{tab:transfer_two_tasks_compact}
    \resizebox{0.405\textwidth}{!}{
    \begin{tabular}{ccc|cccc}
        \toprule
        \multicolumn{3}{c|}{Method} &
        PSNR$\uparrow$ & SSIM$\uparrow$ & KID$\downarrow$ & LPIPS$\downarrow$ \\
        \midrule
        \multicolumn{7}{c}{\textbf{Gaussian Deblurring}} \\
        \midrule
        \multicolumn{3}{l|}{Resample} & 26.47 & \underline{0.755} & 0.101 & 0.157 \\
        \multicolumn{3}{l|}{FlowChef} & 27.59 & \textbf{0.768} & 0.026 & 0.140 \\
        \multicolumn{3}{l|}{FlowDPS} & 26.16 & 0.743 & 0.034 & 0.194 \\
        \multicolumn{3}{l|}{FLAIR} & \underline{27.74} & 0.720 & \textbf{0.012} & \underline{0.109} \\
        \rowcolor{green!10}
        \multicolumn{3}{l|}{$\text{TriPS}_{\text{G}}$ on SR$\times$8} & \textbf{28.90} & 0.745 & \underline{0.014} & \textbf{0.089} \\
        \midrule
        \multicolumn{7}{c}{\textbf{Super-Resolution $\times$12}} \\
        \midrule
        \multicolumn{3}{l|}{Resample} & 22.92 & 0.641 & 0.156 & 0.355 \\
        \multicolumn{3}{l|}{FlowChef} & 26.48 & 0.743 & 0.089 & 0.201 \\
        \multicolumn{3}{l|}{FlowDPS} & 26.93 & 0.756 & 0.020 & 0.147 \\
        \multicolumn{3}{l|}{FLAIR} & \underline{27.51} & \underline{0.765} & \underline{0.017} & \underline{0.148} \\
        \rowcolor{green!10}
        \multicolumn{3}{l|}{$\text{TriPS}_{\text{G}}$ on SR$\times$8} & \textbf{28.80} & \textbf{0.774} & \textbf{0.012} & \textbf{0.099} \\
        \bottomrule
    \end{tabular}}
    \vspace{-1em}
\end{table}

\paragraph{Schedule transferability on degradation shifts}
\label{par:sched_general}
We evaluate whether GRPO-based schedule optimization generalize to different degradation settings. Specifically, we transfer the triadic schedules ($\lambda(t), \beta(t), \eta(t)$) optimized on SR$\times$8 directly to Gaussian deblurring and SR$\times$12. As shown in Table~\ref{tab:transfer_two_tasks_compact}, $\mathrm{TriPS}_{\mathrm{G}}$ outperforms established flow matching baselines across both distortion and perceptual metrics. 

\paragraph{Diffusion model} \label{par:diffusion_model}
We further validate our method within the diffusion based framework, implemented on Stable Diffusion 1.5 ($512^2$, NFE 50), against established solvers including PSLD~\cite{psld}, DDPG~\cite{ddpg}, P2L~\cite{p2l}, and TReg~\cite{treg}, with all baseline hyperparameters carefully tuned and reported in Appendix~\ref{par:diffusion_implementation}. As shown in Table~\ref{tab:quanti_diffusion}, the proposed method consistently achieves superior performance over all competing approaches, with qualitative comparisons provided in Appendix~\ref{app:subsec_trips_diffusion}.

\begin{table}[t!]
    \centering
    \renewcommand{\arraystretch}{0.9} 
    \setlength{\tabcolsep}{0.6em} 
    \caption{Quantitative comparison on linear inverse problems based on diffusion model. Evalutation metrics are averaged over 1,000 samples from the FFHQ dataset using 50 NFEs under additive Gaussian noise ($\sigma_n = 0.03$). For each metric, the best and second-best results are indicated in \textbf{bold} and \underline{underline}, respectively.}
    \resizebox{0.48\textwidth}{!}{
    
        \begin{tabular}{l|cccc|cccc}
        \toprule
        \multicolumn{9}{c}{\textbf{Diffusion Model (SD1.5)}}\\
        \midrule
        & \multicolumn{4}{c|}{Super-Resolution $\times$8} & \multicolumn{4}{c}{Motion Deblurring} \\
        \midrule
        Method &
        PSNR$\uparrow$ & SSIM$\uparrow$ & FID$\downarrow$ & LPIPS$\downarrow$ &PSNR$\uparrow$ & SSIM$\uparrow$ & FID$\downarrow$ & LPIPS$\downarrow$ \\
        \midrule
        PSLD & 21.61 & 0.636 & 86.96 & 0.399 & 21.41 & 0.582 & 84.59 & 0.406 \\
        DDPG & 22.42 & 0.689 & 116.39 & 0.308 & 24.48 & 0.548 & 100.78 & 0.275 \\
        P2L & 25.31 & 0.711 & 95.27 & 0.271 & 25.52 & 0.656 & 55.15 & 0.230 \\
        TReg & 26.34 & 0.733 & 128.19 & 0.254 & 26.36 & 0.669 & 57.94 & 0.200 \\
        \rowcolor{red!10}
        $\text{TriPS}_{\text{T}}$ (Ours) & \underline{27.07} & \underline{0.737} & \underline{52.91} & \underline{0.192} & \underline{28.19} & \underline{0.779} & \underline{32.58} & \underline{0.164} \\
        \rowcolor{green!10}
        $\text{TriPS}_{\text{G}}$ (Ours) & \textbf{28.02} & \textbf{0.782} & \textbf{43.04} & \textbf{0.164} & \textbf{28.39} & \textbf{0.782} & \textbf{26.37} & \textbf{0.162} \\
        \bottomrule
        \end{tabular}
    }
    \label{tab:quanti_diffusion}
\end{table}

\begin{table}[t!]
    \centering
    \renewcommand{\arraystretch}{0.9} 
    \setlength{\tabcolsep}{1.3em} 
    \caption{
    Ablation studies of TriPS evaluated on 100 FFHQ samples. $\text{TriPS}_{\text{G}}$ is optimized via GRPO starting from $\text{TriPS}_{\text{T}}$ baseline.
    }
    \resizebox{0.45\textwidth}{!}{
    
        \begin{tabular}{ccc|cccc}
        \toprule
        \multicolumn{7}{c}{\textbf{(a) Reward Guided Perception-Distortion Controller}}\\
        \midrule
        \multicolumn{3}{l|}{Case} &
        PSNR$\uparrow$ & SSIM$\uparrow$ & KID$\downarrow$ & LPIPS$\downarrow$ \\
        \midrule
        \multicolumn{3}{l|}{$\text{TriPS}_{\text{G}}^{\text{Dist}}$} & \textbf{30.62} & \textbf{0.840} & 0.036 & 0.072 \\
        \midrule
        \multicolumn{3}{l|}{$\text{TriPS}_{\text{T}}$} & \underline{30.33} & \underline{0.822} & \underline{0.016} & \underline{0.070} \\
        \midrule
        \multicolumn{3}{l|}{$\text{TriPS}_{\text{G}}^{\text{Perc}}$} & 29.89 & 0.793 & \textbf{0.011} & \textbf{0.067} \\
        \midrule
        \multicolumn{7}{c}{\textbf{(b) Joint Triadic Schedule Optimization}}\\
        \midrule
        DC & CFG & Stoch. &
        PSNR$\uparrow$ & SSIM$\uparrow$ & KID$\downarrow$ & LPIPS$\downarrow$ \\
        \midrule
        \xmark & \xmark & \xmark & 29.04 & 0.783 & 0.121 & 0.155 \\
        \cmark & \xmark & \xmark & 29.43 & \underline{0.808} & \underline{0.117} & \underline{0.125} \\
        \xmark & \cmark & \xmark & 29.05 & 0.784 & 0.121 & 0.154 \\
        \xmark & \xmark & \cmark & \underline{29.51} & 0.804 & 0.120 & 0.133 \\
        \rowcolor{red!10}
        \cmark & \cmark & \cmark & \textbf{29.77} & \textbf{0.820} & \textbf{0.072} & \textbf{0.101} \\
        \bottomrule
        \end{tabular}
    }
    \vspace{-1.0em}
    \label{tab:ablation_studies}
\end{table}

\subsection{Ablation Studies}
\label{subsec:ablation_studies}
\paragraph{Reward guided control of the perception–distortion trade-off}
The GRPO-based optimization navigates the perception–distortion trade-off by modulating weights in the reward function $R = w_{\text{dist}}R_{\text{dist}} + w_{\text{perc}}R_{\text{perc}}$. 
Initialized from $\text{TriPS}_{\text{T}}$ schedules, we optimize two variants for Gaussian deblurring on FFHQ dataset: $\text{TriPS}_{\text{G}}^{\text{Dist}}$ (($w_{\text{dist}}=0.9, w_{\text{perc}}=0.1$)) and $\text{TriPS}_{\text{G}}^{\text{Perc}}$ ($w_{\text{dist}}=0.3, w_{\text{perc}}=0.7$). As shown in Fig.~\ref{fig:grpo_results_overall}(a), each variant successfully maximizes its targeted reward while the alternate reward decreases during optimization. This aligns with the results in Table~\ref{tab:ablation_studies}(a), confirming that each variant effectively enhances performance in its intended direction. Furthermore, these trends are visually reflected in Fig.~\ref{fig:grpo_results_overall}(b), where $\text{TriPS}_{\text{G}}^{\text{Perc}}$ restores sharper textures and $\text{TriPS}_{\text{G}}^{\text{Dist}}$ preserves high data fidelity without structural artifacts. The resulting triadic schedules for DC guidance, CFG, and stochasticity are presented in Fig.~\ref{fig:grpo_results_overall}(c). Detailed analysis of these curves are provided in the Appendix~\ref{app:subsec_ablation_more_discussion_PD_control}.

\paragraph{Impact of joint triadic schedule optimization}
To validate the necessity of joint triadic optimization, we compare our fully scheduled $\text{TriPS}_{\text{T}}$ approach against baselines where only one component, such as DC guidance, CFG, or stochasticity, is optimized while the others are fixed to its search range's mean. Evaluated on the SR $\times8$ task, the results in Table~\ref{tab:ablation_studies}(b) demonstrate that the the simultaneous optimization of all three temporal schedules consistently yields superior performance compared to the partially fixed variants. These results show that leveraging their interplay in time-varying scheduling plays a constructive role in enhancing restoration quality. Additional results are provided in Appendix~\ref{app:subsec_ablation_more_joint}.

\begin{figure}[t!]
  \centering
  \includegraphics[width=0.48\textwidth]{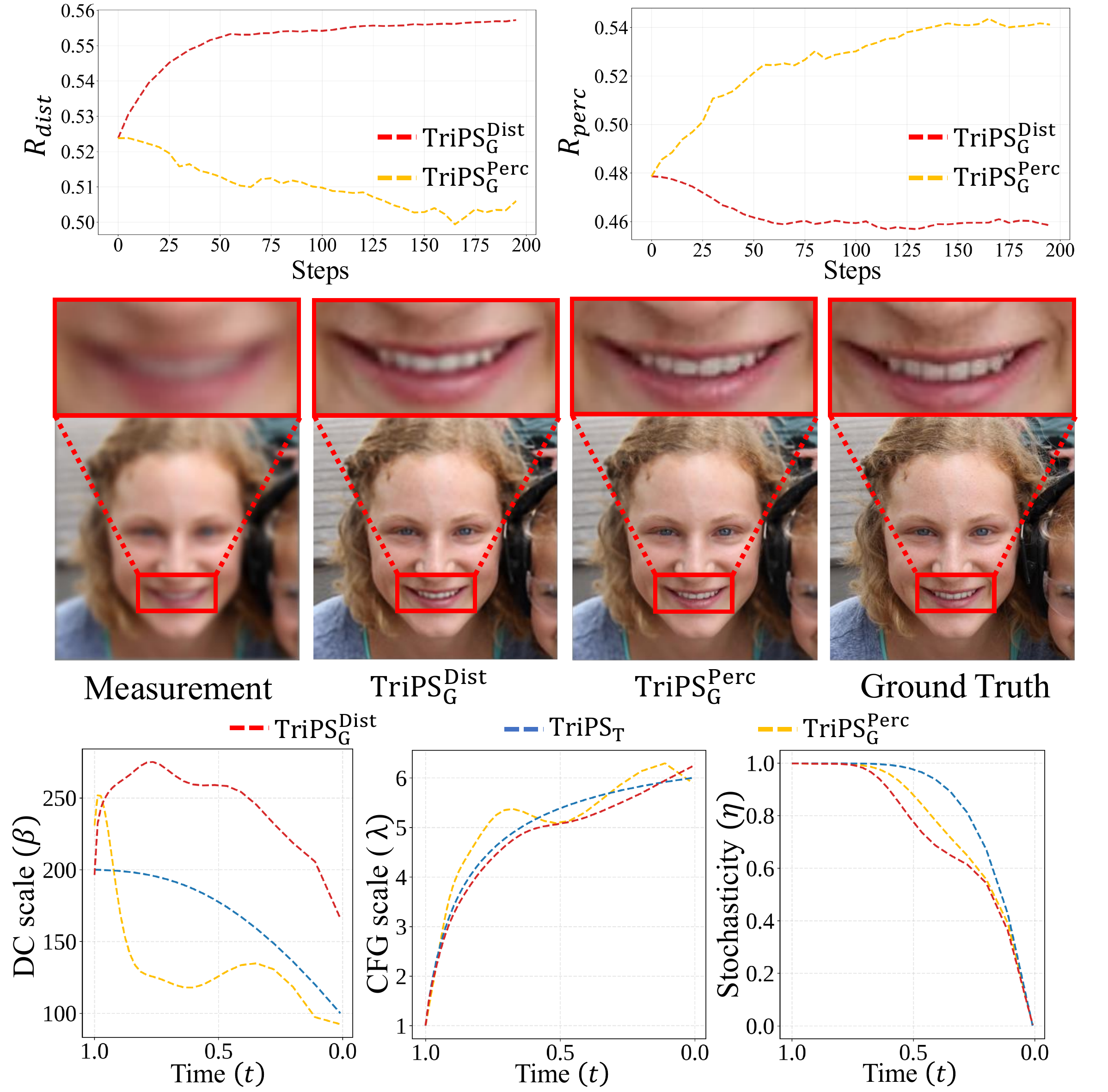}
  \caption{
  Reward-guided control of the perception-distortion trade-off via $\text{TriPS}_{\text{G}}$. (Top) Evolution of distortion and perceptual rewards ($R_{\text{dist}}$ and $R_{\text{perc}}$) on a validation set during optimization. 
  (Middle) Visual comparison between the perception oriented ($\text{TriPS}_{\text{G}}^{\text{Perc}}$) and distortion oriented ($\text{TriPS}_{\text{G}}^{\text{Dist}}$) variants, both optimized via GRPO starting from the $\text{TriPS}_{\text{T}}$ baseline.
  $\text{TriPS}_{\text{G}}^{\text{Perc}}$ recovers sharper textures, whereas the $\text{TriPS}_{\text{G}}^{\text{Dist}}$ preserves structural fidelity without artifacts. (Bottom) The optimized triadic schedules for DC guidance, CFG, and stochasticity across the three cases.
  }
  \vspace{-1em}
  \label{fig:grpo_results_overall}
\end{figure}

\subsection{Discussion}
The optimized $\text{TriPS}_{\text{G}}$ schedules align with the proposed triadic scheduling trends in Sec.~\ref{subsec:tri_sched_trend}. Beyond the consistent global patterns, the optimized schedules exhibit fine-grained temporal variations, such as local non-monotonic fluctuations and magnitude shifts, that prove critical for pushing the perception-distortion Pareto frontier. This effectiveness extends beyond flow matching frameworks to diffusion models (See Appendix~\ref{app:subsec_trips_diffusion}), demonstrating TriPS generalizability across generative architectures. Furthermore, applying TriPS to diverse solvers including FlowDPS~\cite{flowdps} and FLAIR~\cite{flair} consistently improves perceptual metrics while maintaining competitive distortion performance compared to default configurations, confirming the broad compatibility of triadic schedule optimization across inverse problem solvers (See Appendix~\ref{app:subsec_applicability_trips}). 

\section{Conclusion}
\label{sec: conc}
In this work, we propose TriPS, a framework for posterior sampling through time-varying coordination of data consistency (DC) guidance, classifier-free guidance (CFG), and stochasticity. We characterize a triadic coupling dynamic in which excessive CFG scales lead to the guidance conflict that hinder data consistency convergence, while stochasticity regularizes off-manifold phenomena induced by DC guidance and CFG. To optimize these dynamics, two complementary methods are proposed: template-based schedule search that employs fixed functional templates to perform coarse schedule search, providing robust baselines across inverse problems, and GRPO-based schedule optimization that captures complex temporal curves beyond fixed functional forms, enabling effective navigation of the perception-distortion trade-off through a hybrid IQA reward that jointly regularizes both metrics. Through these strategies, TriPS achieves high fidelity restoration across diverse tasks. The established triadic scheduling principles will provide a robust basis for future research in adaptive control for generative modeling and its broader applications.





\section*{Acknowledgements}
This work was supported in part by Institute of Information \& Communications Technology Planning \& Evaluation (IITP) grants funded by the Korea government(MSIT) [No.~RS-2021-II211343, Artificial Intelligence Graduate School Program (Seoul National University) / No.~RS-2025-02314125, Effective Human-Machine Teaming With Multimodal Hazy Oracle Models], the National Research Foundation of Korea (NRF) grant funded by the Korea government(MSIT) (No.~RS-2025-02263628), the BK21 FOUR program of the Education and Research Program for Future ICT Pioneers, Seoul National University, and Samsung Electronics Co., Ltd.\ (IO251217-14748-01).


\section*{Impact Statement}
Our research introduces TriPS (Triadic Dynamics Aware Posterior Sampling), a novel framework that optimizes the time-varying interplay of data consistency (DC) guidance, classifier-free guidance (CFG), and stochasticity in generative posterior sampling for solving inverse problems with pretrained diffusion and flow matching models. We develop two complementary schedule optimization strategies: template-based schedule search using compact functional priors and GRPO-based schedule optimization strategies that discovers complex temporal curves. These advancements offer broad practical utility in fields such as medical imaging, scientific reconstruction, and remote sensing, while providing new theoretical and algorithmic insights for conditional generative sampling. This work raises no ethical concerns and focuses on enhancing reconstruction accuracy without adverse societal consequences.





\bibliography{main}
\bibliographystyle{icml2026}

\newpage
\appendix
\onecolumn
\section{Derivations and Proof}
\label{app:sec_derivations_and_proof}

This appendix derives (i) the flow-based posterior-mean estimator $\hat{\bm{x}}_{0|t}(\bm{x}_t)\approx\mathbb{E}[x_0\mid \bm{x}_t]$ used for likelihood/DC guidance, and (ii) the corresponding relation between the marginal score $\nabla_{\bm{x}_t}\log p_t(\bm{x}_t)$ and the flow velocity field. The derivations follow the standard OT/linear conditional flow setting.

\subsection{Setup: linear conditional flow}
\label{app:subsec_linear_conditional_flow}
We consider the linear conditional flow (OT/rectified flow matching) that interpolates between clean data $\bm{x}_0\sim p_0$ and a Gaussian endpoint $\bm{x}_1\sim \mathcal{N}(0,I)$:
\begin{equation}
\label{eq:app_linear_flow}
\bm{x}_t = (1-t)\bm{x}_0 + t x_1,\qquad t\in[0,1].
\end{equation}
The (optimal) marginal velocity field is given by the conditional expectation
\begin{equation}
\label{eq:app_marginal_velocity}
v_t(\bm{x}) \triangleq \mathbb{E}[\bm{x}_1-\bm{x}_0\mid \bm{x}_t=\bm{x}],
\end{equation}
which is approximated by the learned model velocity $v_\theta(\bm{x}_t;\emptyset)$ (time dependence is implicit, consistent with the main text).

\subsection{Flow-based Posterior Mean}
\label{app:subsec_flow_tweedie_mean}
From \eqref{eq:app_linear_flow}, we have $\bm{x}_t-\bm{x}_0=t(\bm{x}_1-\bm{x}_0)$, hence
\begin{equation}
\label{eq:app_x0_identity}
\bm{x}_0 = \bm{x}_t - t(\bm{x}_1-\bm{x}_0).
\end{equation}
Taking conditional expectation with respect to $p_t(x_0,x_1\mid x_t)$ yields the flow-based analogue of Tweedie's variable identity:
\begin{equation}
\label{eq:app_tweedie}
\mathbb{E}[\bm{x}_0\mid \bm{x}_t]
= \bm{x}_t - t\,\mathbb{E}[\bm{x}_1-\bm{x}_0\mid \bm{x}_t]
= \bm{x}_t - t\,v_t(\bm{x}_t).
\end{equation}
Replacing $v_t$ with the learned velocity gives the estimator used in the main text:
\begin{equation}
\label{eq:app_xhat0}
\hat{\bm{x}}_{0|t}(\bm{x}_t)\triangleq \mathbb{E}[\bm{x}_0\mid \bm{x}_t]\approx \bm{x}_t - t\,v_\theta(\bm{x}_t;\emptyset).
\end{equation}
Similarly, the conditional expectation of the Gaussian endpoint can be written as
\begin{equation}
\label{eq:app_xhat1}
\mathbb{E}[\bm{x}_1\mid \bm{x}_t]=\bm{x}_t + (1-t)\,v_t(\bm{x}_t).
\end{equation}

\subsection{Score from Velocity}
\label{app:subsec_score_from_velocity}
Under \eqref{eq:app_linear_flow} with $\bm{x}_1\sim\mathcal{N}(0,I)$, the conditional distribution is
\begin{equation}
\label{eq:app_conditional}
p_t(\bm{x}_t\mid \bm{x}_0)=\mathcal{N}\big((1-t)\bm{x}_0,\; t^2 I\big),
\end{equation}
and thus the conditional score is
\begin{equation}
\label{eq:app_conditional_score}
\nabla_{\bm{x}_t}\log p_t(\bm{x}_t\mid \bm{x}_0)
= -\frac{1}{t^2}\big(\bm{x}_t-(1-t)\bm{x}_0\big).
\end{equation}
Using $\nabla_{\bm{x}_t}\log p_t(\bm{x}_t)=\mathbb{E}[\nabla_{\bm{x}_t}\log p_t(\bm{x}_t\mid \bm{x}_0)\mid \bm{x}_t]$ gives
\begin{equation}
\label{eq:app_marginal_score_condexp}
\nabla_{\bm{x}_t}\log p_t(\bm{x}_t)
= -\frac{1}{t^2}\Big(\bm{x}_t-(1-t)\mathbb{E}[\bm{x}_0\mid \bm{x}_t]\Big).
\end{equation}
Substituting \eqref{eq:app_tweedie} yields:
\begin{equation}
\label{eq:app_score_velocity}
\nabla_{\bm{x}_t}\log p_t(\bm{x}_t)
= -\frac{\bm{x}_t}{t} - \frac{1-t}{t}\,v_t(\bm{x}_t)
\ \approx\ -\frac{\bm{x}_t}{t} - \frac{1-t}{t}\,v_\theta(\bm{x}_t;\emptyset).
\end{equation}
Equivalently, solving \eqref{eq:app_score_velocity} for $v_t$ gives
\begin{equation}
\label{eq:app_velocity_score}
v_t(\bm{x}_t)= -\frac{\bm{x}_t}{1-t}-\frac{t}{1-t}\,\nabla_{\bm{x}_t}\log p_t(\bm{x}_t),
\end{equation}
which matches the form used in \cite{liu2025flow}.

\subsection{Proof of Proposition~\ref{prop:cfg_slows_dc}}
\label{app:subsec_proof_cfg_slows_dc}
\begin{proof}
In Proposition~\ref{prop:cfg_slows_dc}, we conduct a first-order derivative analysis.
To analyze the derivative of the one-step expected residual norm $\mathcal{R}(\hat{\bm{x}}_{0|t}) \triangleq \|y - \mathcal{A}\hat{\bm{x}}_{0|t}\|_2^2$ reduction with respect to the CFG scale $\lambda(t)$, we start from the Euler-Maruyama discretization in~\cite{platen1992numerical}:
\begin{equation}
\label{app_eq:em_step}
\begin{aligned}
\bm{x}_{t+\Delta t}
=
\bm{x}_t
+\Delta t\,b_{\text{prior}}(\bm{x}_t)
+\Delta t\,b_{\text{cfg}}(\bm{x}_t;c) 
+\Delta t\,b_{\text{dc}}(\bm{x}_t;\bm{y})
+\ b_{\text{sto}}(t)\sqrt{\Delta t}\,\xi,
\end{aligned}
\end{equation}
where $\xi \sim \mathcal{N}(0, I)$ and $\tilde{b}_{\text{cfg}}$ is the unit-scale CFG defined in \eqref{eq:drfits_definition} and $b_{\mathrm{sto}}(t)$ is the stochastic drift coefficient
corresponding to the stochasticity scale $\eta(t)$ in Eq.~(5). Using the second-order Taylor expansion around $\hat{\bm{x}}_{0|t}$, the residual norm at the next state $\hat{\bm{x}}_{0|t+\Delta t}$ is:
\begin{equation}
\label{app_eq:taylor}
\begin{aligned}
\mathcal{R}(\hat{\bm{x}}_{0|t+\Delta t})) &= \mathcal{R}(\hat{\bm{x}}_{0|t}) + \nabla \mathcal{R}(\hat{\bm{x}}_{0|t})^\top (\hat{\bm{x}}_{0|t+\Delta t} - \hat{\bm{x}}_{0|t}) \\
&+ \frac{1}{2} (\hat{\bm{x}}_{0|t+\Delta t} - \hat{\bm{x}}_{0|t})^\top \nabla^2 \mathcal{R}(\hat{\bm{x}}_{0|t}) (\hat{\bm{x}}_{0|t+\Delta t} - \hat{\bm{x}}_{0|t}) + o(|\hat{\bm{x}}_{0|t+\Delta t} - \hat{\bm{x}}_{0|t}|^2).
\end{aligned}
\end{equation}
We now compute the conditional expectation $\mathbb{E}[\mathcal{R}(\hat{\bm{x}}_{0|t+\Delta t}) \mid \bm{x}_t]$. Substituting \eqref{app_eq:em_step} into \eqref{app_eq:taylor} and noting that $\mathbb{E}[\xi] = 0$ and $\mathbb{E}[\xi \xi^\top] = I$, we obtain:
\begin{equation}
\label{app_eq:expectation}
\begin{aligned}
\mathbb{E}[\mathcal{R}(\hat{\bm{x}}_{0|t+\Delta t}) \mid \bm{x}_t] &= \mathcal{R}(\hat{\bm{x}}_{0|t}) + \Delta t \nabla \mathcal{R}(\hat{\bm{x}}_{0|t})^\top \Big( b_{\text{prior}} + \lambda(t)\tilde{b}_{\text{cfg}}(\bm{x}_t;c) + b_{\text{dc}}(\bm{x}_t;\bm{y}) \Big) \\
&+ \frac{1}{2} b_{\text{sto}}(t)^2 \Delta t \text{Tr}\big(\nabla^2 \mathcal{R}(\hat{\bm{x}}_{0|t})\big) + o(\Delta t).
\end{aligned}
\end{equation}
In this derivation, the quadratic terms involving $b_{\text{det}}(\bm{x}_t)$ and $\lambda(t)$-dependence of $b_{\text{dc}}$ are of order $O(\Delta t^2)$ and are absorbed into the $o(\Delta t)$ remainder neglecting higher-order curvature effect. Finally, to find the first-order derivative with respect to the CFG scale $\lambda(t)$, we differentiate the expected measurement residual in \eqref{app_eq:expectation}:
\begin{equation}
\label{app_eq:final_diff}
\begin{aligned}
\frac{\partial}{\partial \lambda(t)} \mathbb{E}[\mathcal{R}(\hat{\bm{x}}_{0|t+\Delta t} \mid \bm{x}_t] &= \frac{\partial}{\partial \lambda(t)} \left[ \Delta t \nabla \mathcal{R}(\hat{\bm{x}}_{0|t})^\top \Big( \lambda(t)\tilde{b}_{\text{cfg}}(\bm{x}_t;c) \Big) \right] + o(\Delta t) \\
&= \Delta t \Big\langle \nabla \mathcal{R}(\hat{\bm{x}}_{0|t}), \tilde{b}_{\text{cfg}}(\bm{x}_t;c) \Big\rangle + o(\Delta t) \\
&= -\Delta t\,\Big\langle \tilde b_{\mathrm{dc}}(\bm{x}_t;\bm{y}),\,\tilde b_{\mathrm{cfg}}(\bm{x}_t;c)\Big\rangle
+o(\Delta t),
\end{aligned}
\end{equation}
where $\Big\langle \nabla \mathcal{R}(\hat{\bm{x}}_{0|t}), \tilde{b}_{\text{cfg}}(\bm{x}_t;c) \Big\rangle \approx -\Big\langle \tilde b_{\mathrm{dc}}(\bm{x}_t;\bm{y}),\,\tilde b_{\mathrm{cfg}}(\bm{x}_t;c)\Big\rangle$ by our experimental validation.
Our experimental validation is set by cosine similarity ($\text{Cos-Sim($-\mathcal{R}(\hat{\bm{x}}_{0|t})$, $\tilde b_{\mathrm{dc}}(\bm{x}_t;\bm{y})$)}$) between $-\mathcal{R}(\hat{\bm{x}}_{0|t})$ and $\tilde b_{\mathrm{dc}}(\bm{x}_t;\bm{y})$. Empirically, in Fig.~\ref{fig:prop_3.1_validation}, we observe that the cosine similarity between $-\mathcal{R}(\hat{\bm{x}}_{0|t})$ and $\tilde b_{\mathrm{dc}}(\bm{x}_t;\bm{y})$ remains near unity throughout the entire sampling trajectory, thereby substantiating the approximation in Eq.~\eqref{app_eq:final_diff}.

\begin{figure}[ht]
    \centering
    \includegraphics[width=0.40\linewidth]{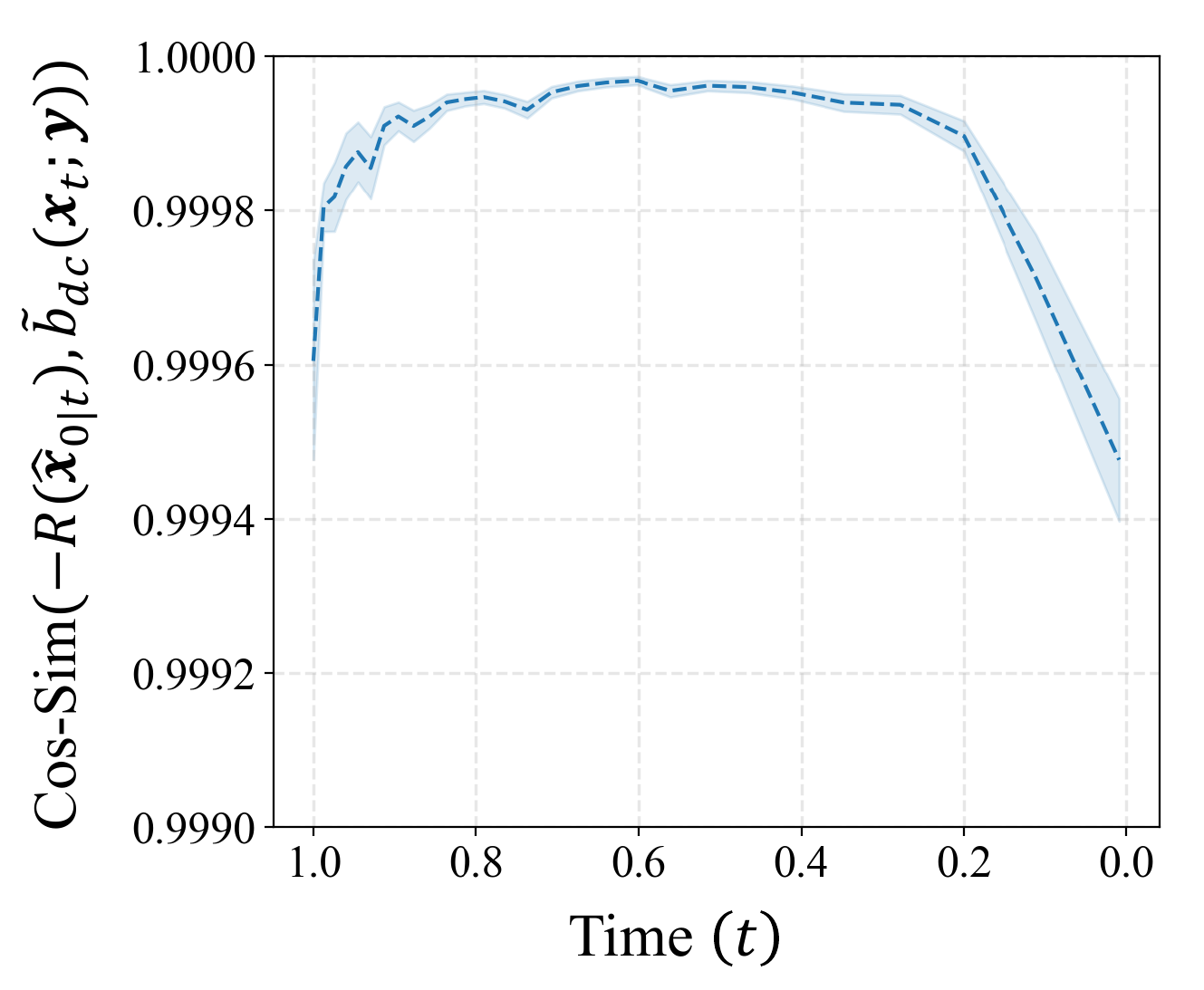}
    \caption{Cosine similarity between squared residual norm $-\mathcal{R}(\hat{\bm{x}}_{0|t})$ and unit-scale DC guidance $\tilde b_{\mathrm{dc}}(\bm{x}_t;\bm{y})$.}
    \label{fig:prop_3.1_validation}
\end{figure}

Conclusively, this results shows that the directional alignment (inner product) between the DC guidance and the CFG determines whether increasing the CFG scale $\lambda(t)$ accelerates or hinders the minimization of the measurement residual.
\end{proof}

\section{More Analysis on Triadic Coupling Dynamics}
\label{app:sec_drift_interplay_and_regularization}
\subsection{Sampling Process: Early-Stage Guidance Conflict}
\label{app:subsec_sampling_process_conflict_dc_and_cfg}

To empirically substantiate the theoretical analysis presented in Sec.~\ref{subsec:dc_cfg_conflict}, we visualize the evolution of the clean image prediction $\hat{\bm{x}}_{0|t}$ defined in Eq.~\eqref{eq:tweedie_update} throughout the reverse diffusion process. Fig.~\ref{fig:process_comparison_dc_and_cfg} provides a comparative visualization of the restoration trajectories under a lower CFG scale ($\lambda=2.0$) and a high CFG scale ($\lambda=7.0$).

\begin{figure}[ht]
    \centering
    \includegraphics[width=0.90\linewidth]{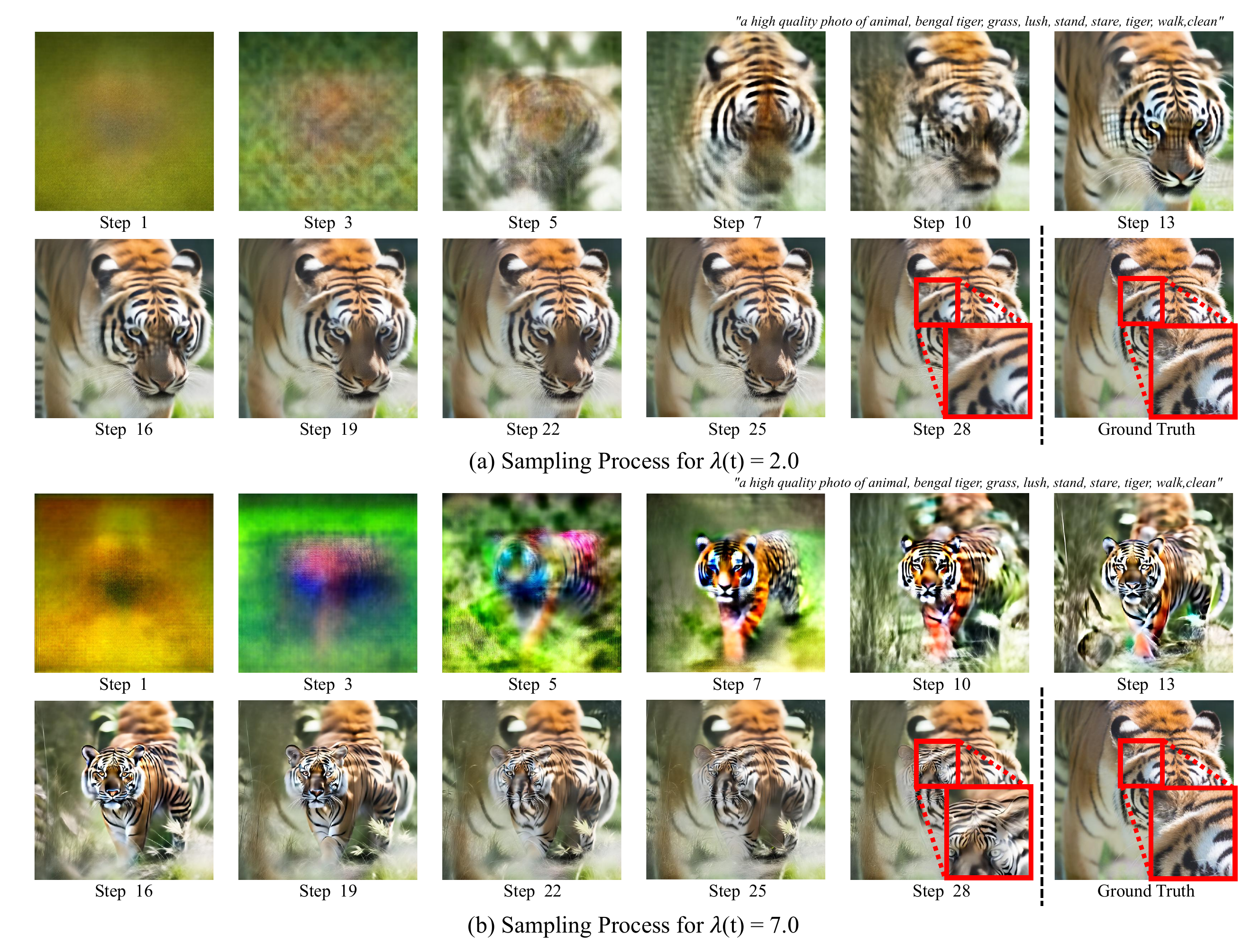}
    \caption{Visual comparison of early-stage dynamics under different CFG scales at different sampling steps (NFE:28). This experiment demonstrates a super-resolution $\times8$ on the DIV2K validation set. (a): Low CFG ($\lambda(t)=2.0$) scale facilitates a gradual transition from noise to clean data, respecting data consistency. (b): High CFG ($\lambda(t)=7.0$) scale induces sudden shifts and intense color saturation in the early phase. This visualizes the \textit{guidance conflict}, where the aggressive CFG steers the generation toward semantic hallucinations that violate the measurement constraints.}
    \label{fig:process_comparison_dc_and_cfg}
\end{figure}

The visual results corroborate our theoretical claims. Consistent with the behavior of standard CFG scale reported in recent studies~\citep{cfg++,Imagen}, the high CFG scale (Fig.~\ref{fig:process_comparison_dc_and_cfg} (b)) exhibits rapid signal amplification and saturation during the early timesteps. As formalized in Eq.~\eqref{eq:cos_sim_def}, this is the phase where the CFG $\tilde b_{\mathrm{CFG}}(\bm{x}_t;c)$ and the DC guidance $\tilde b_{\mathrm{DC}}(\bm{x}_t;\bm{y})$ are most misaligned ($\text{COS-SIM}_{1}(\bm{x}_t) < 0$). 

Specifically, the high CFG scale (Fig.~\ref{fig:process_comparison_dc_and_cfg} (b)) forces the intermediate states to commit prematurely to a semantic mode, leading to hallucinations that are structurally inconsistent with the degraded measurements $\bm{y}$. In contrast, the low CFG scale (Fig.~\ref{fig:process_comparison_dc_and_cfg} (a)) maintains a trajectory that allows the data consistency term to effectively guide the restoration process without conflicting with the prior. These observations confirm that the conflict is not merely theoretical but manifests as tangible degradation in the early-stage generation process.

\subsection{Sampling Process: Stochasticity as a Regularizer for DC guidance and CFG}
\label{app:subsec_sampling_process_stochasticity_regularization}
To further substantiate the observation in Sec.~\ref{subsec:analysis_sto_regularize}, we conduct a process-wise analysis to visualize how stochasticity mitigates the instability induced by strong DC guidance and CFG ($b_{\text{dc}}(\bm{x}_t;\bm{y}), b_{\text{cfg}}(\bm{x}_t; c)$). Following the setup in Fig.~\ref{fig:analysis_stochasticity}, we conduct this analysis scheduling the stochasticity scale $\eta(t) = \zeta(t)\sqrt{1-\sigma_{t+\Delta t}}$. While Fig.~\ref{fig:process_comparison_sto} illustrates the final-state $t = t_{0}$, the intermediate sampling dynamics provide deeper insights into the "regularization" role of stochasticity. As shown in Fig.~\ref{fig:process_comparison_sto}, employing high DC, CFG scales $\beta(t), \lambda(t)$ without sufficient stochasticity leads to a progressive accumulation of artifacts. High DC, CFG scales push the intermediate latent $x_t$ toward the data consistency constraint with such intensity that it overshoots the natural image manifold, resulting in the high-frequency "saturated" artifacts observed in Fig.~\ref{fig:process_comparison_sto} (a). In contrast, as shown in the step-by-step evolution in Fig.~\ref{fig:process_comparison_sto} (b), the introduction of a proper scheduled stochasticity acts as a per-step manifold projection. By injecting controlled stochasticity and subsequently removing it via the reverse diffusion step, the sampler effectively repairs the accumulated error, pulling the trajectory back toward the high-density regions of the prior. This process-wise evidence confirms that the synergetic coupling of scales of DC guidance, CFG and stochasticity ($\beta(t), \lambda(t)$, $\eta(t)$) is not merely hyperparameters balancing act, but a necessary mechanism to maintain sampling stability under strong guidance.

\begin{figure}[ht]
    \centering
    \includegraphics[width=0.90\linewidth]{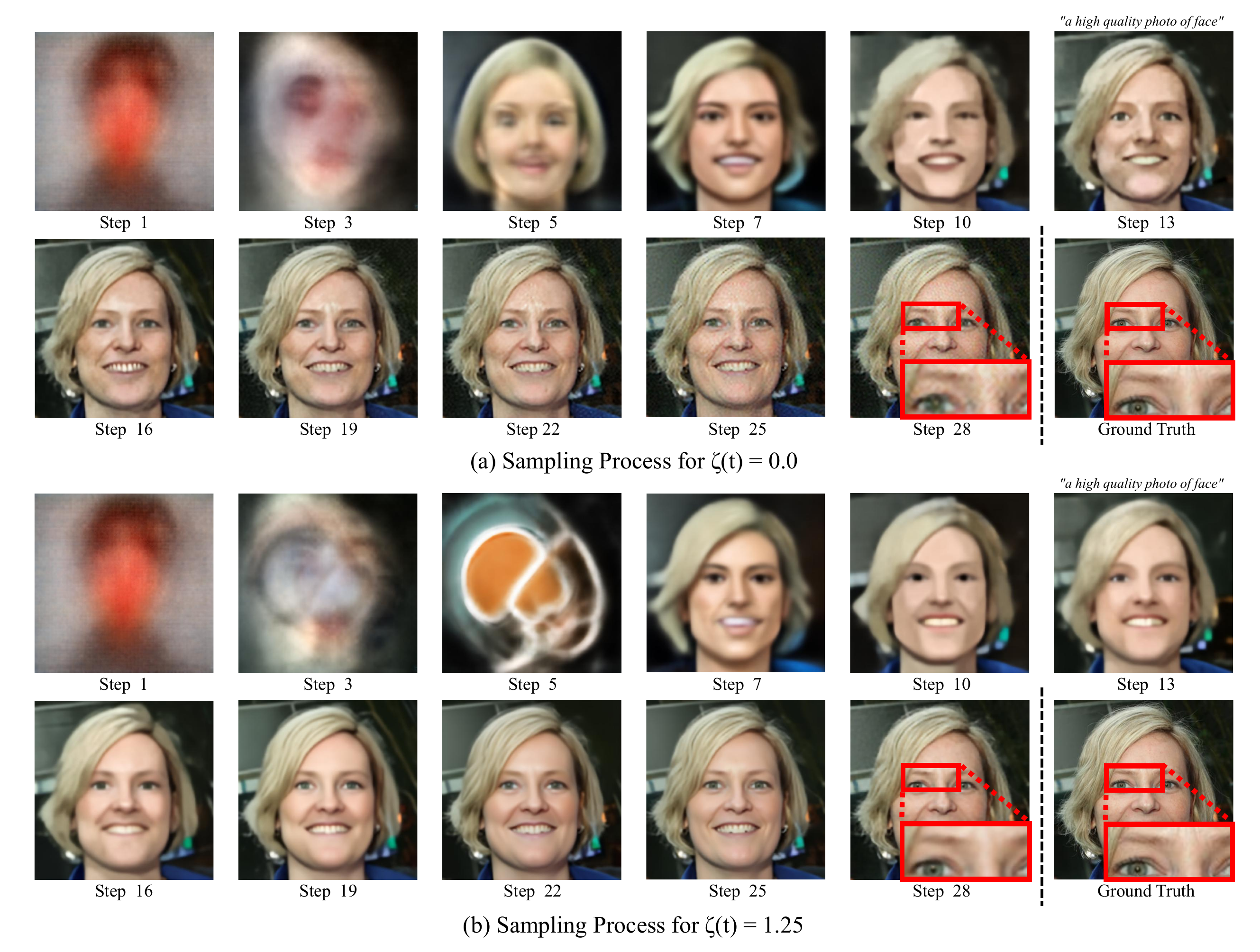}
    \caption{Visual comparison of the stabilization effect of stochasticity under different stochasticity scales at different sampling steps (NFE:28). The stochasticity scale is scheduled as $\eta(t) = \zeta(t)\sqrt{1-\sigma_{t+\Delta t}}$ for this analysis. This experiment demonstrates a super-resolution $\times8$ on the FFHQ validation set. (a): Without stochasticity ($\eta(t) = 0$), the error accumulates over iterations, leading to significant manifold deviation and visual artifacts. (b): High stochasticity ($\eta(t) = 1.25$) scale stabilizes the data manifold trajectory, effectively suppressing artifact growth and maintaining perceptual fidelity throughout the restoration process.}
    \label{fig:process_comparison_sto}
\end{figure}

\subsection{Stochasticity as a Regularizer w.r.t Marginal Distribution}
\label{app:subsec_distribution_anal_KID}
To rigorously quantify the marginal distribution discussed in Sec.~3.2, we employ the Kernel Inception Distance (KID) \cite{kid} as a metric for manifold alignment. Unlike the Fréchet Inception Distance (FID), KID is an unbiased estimator even for small sample sizes, making it particularly suitable for our evaluation on the DIV2K and FFHQ datasets (100 samples each).

KID measures the squared Maximum Mean Discrepancy (MMD) between the Inception-v3 feature representations of the generated samples $\hat{\bm{x}}_0$ and the ground-truth images $\bm{x}_{\text{gt}}$:\begin{equation}\text{KID}(\mathbb{P}_g, \mathbb{P}r) = \mathbb{E}{\bm{x}, \bm{x}' \sim \mathbb{P}_g, \bm{y}, \bm{y}' \sim \mathbb{P}_r} [k(\bm{x}, \bm{x}') + k(\bm{y}, \bm{y}') - 2k(\bm{x}, \bm{y})],\end{equation}where $k(\cdot, \cdot)$ denotes the polynomial kernel.

As shown in Fig.~\ref{fig:app_anal_sto_KID}(a), we observe a distinct correlation between the DC,CFG scales ($\lambda, \beta$) and the KID score. When stochasticity is low ($\eta \approx 0$), increasing either the CFG scale $\lambda$ or the DC guidance scale $\beta$ results in an increase of the KID score. This empirical evidence supports our hypothesis that excessive DC guidance and CFG steer the sampling trajectory away from the natural image manifold, inducing a distribution shift. Conversely, as we inject more stochasticity (increasing $\eta$), the KID scores consistently decrease across all DC,CFG scales. This result validates the role of stochasticity as a distributional regularizer that mitigates the bias induced by the triadic coupling of DC guidance and CFG, ensuring that the generated samples remain anchored to the higher-probability regions of the data distribution.

\begin{figure}[ht]
    \centering
    \includegraphics[width=0.60\linewidth]{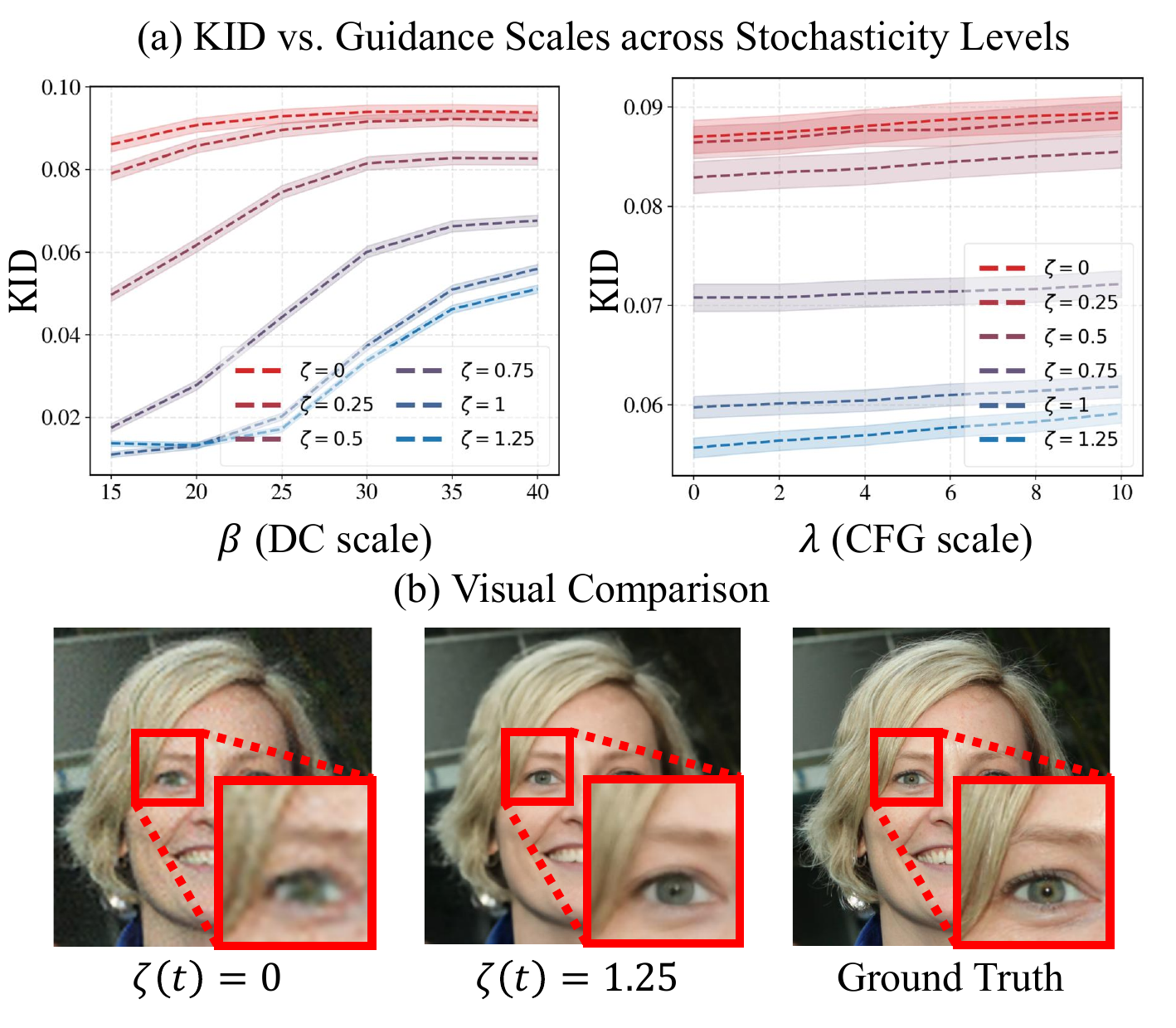}
    \caption{Distributional analysis of the triadic coupling. (a) Kernel Inception Distance (KID) calculated between 100 generated samples and ground-truth images on the DIV2K dataset. Note that the stochasticity scale is scheduled as $\eta(t) = \zeta(t)\sqrt{1-\sigma_{t+\Delta t}}$ for this analysis. The plots demonstrate that high DC guidance,CFG scales ($\lambda, \beta$) without sufficient stochasticity ($\eta$) lead to significant distributional shift away (higher KID). Increasing the stochasticity scale effectively regularizes, pulling the samples back toward the natural image manifold. This highlights the necessity of the triadic balance for maintaining both data consistency and perceptual fidelity.}
    \label{fig:app_anal_sto_KID}
\end{figure}

\subsection{Study on CFG Scheduling Trend in Posterior Sampling}
\label{app:cfg_scheduling_profiles}
We present a qualitative comparison of different CFG scaling schedules $\lambda(t)$ to validate the analysis in Sec.~\ref{subsec:dc_cfg_conflict} regarding the guidance conflict. Specifically, we evaluate distinct schedules within the range $[\lambda_{\min}, \lambda_{\max}] = [1, 6]$: Fixed (constant at the mean), Linearly (increasing/decreasing), and non-monotonic Tent and V-shape function~\cite{smith2017cyclical} curves (linearly increasing to decreasing / linearly decreasing to increasing) for the super-resolution $\times 12$. our observations confirm that a linearly decreasing schedule, which imposes high CFG scales during the early sampling phase, induces the directional misalignment between
DC guidance and CFG, leading to severe semantic hallucinations and artifacts. In contrast, the linearly increasing schedule mitigates this early-stage conflict while preserving the capability for late-stage texture refinement, yielding reconstructions that are most faithful to the ground truth. It is worth noting, however, that while an increasing profile empirically favors fidelity in this regime, the optimal $\lambda(t)$ remains a design choice governed by the trade-off between semantic adherence and data consistency.

\begin{figure}[ht]
    \centering
    \includegraphics[width=0.70\linewidth]{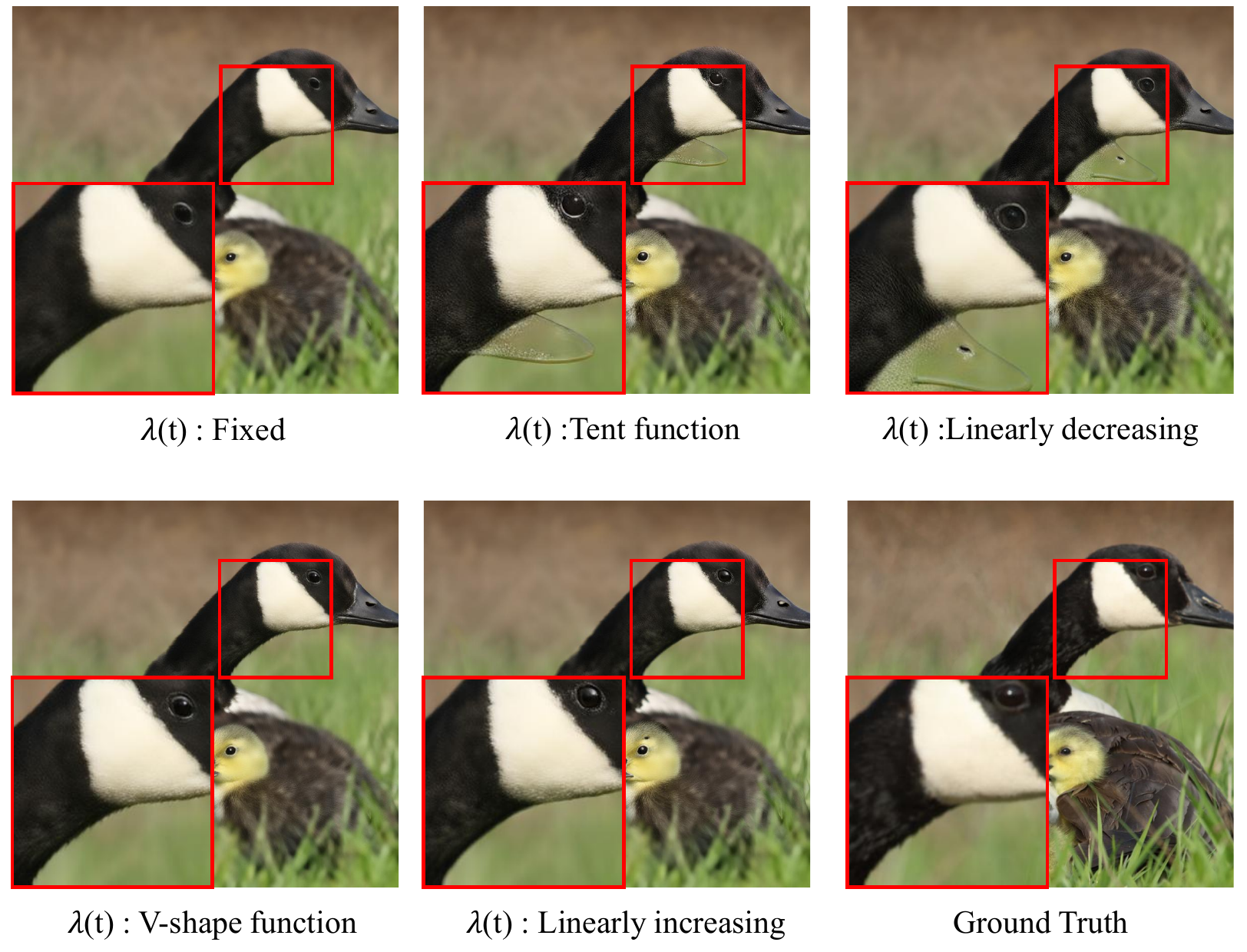}
    \caption{Qualitative comparison of CFG scheduling designs on super-resolution $\times 12$. Using the prompt ``\textit{a photo of baby, bird, duck, duckling, goose, grass, grassy, green, lush, nest, sit, stand, clean}'', we compare reconstruction quality across different $\lambda(t)$ designs. The linearly decreasing schedule induces semantic hallucinations due to the early-stage DC guidance-CFG conflict, whereas the linearly increasing schedule effectively aligns with the ground truth by delaying affects of CFG until the structural feature is established.}
    \label{fig:app_time_window_cfg}
\end{figure}

\section{Details on TriPS Backbone Sampler}
\label{app:sec_algo}
TriPS decomposes the guided reverse step into three core components: generative prior denoising modulated by the classifier-free guidance (CFG) scale $\lambda(t)$, data consistency (DC) guidance refinement with strength $\beta(t)$, and stochastic renoising controlled by $\eta(t)$. To implement the DC guidance, we adopt the hybrid update scheme~\cite{ddpg} within the latent space of the VAE, which interpolates between Back-Projection (BP) and Least-Squares (LS) objectives. This is achieved by minimizing the following loss function:
\begin{equation}
\begin{aligned}
\mathcal{L}(\mathcal{A}\mathcal{D}(\hat{\bm{z}}_{0|t}), \bm{y}) = (1-\omega_t)\,\big\|\mathcal{A}^\dagger(\bm{y})-\mathcal{A}^\dagger\!\big(\mathcal{A}\mathcal{D}(\hat{\bm{z}}_{0|t})\big)\big\|_2 + \omega_t\,\big\|\bm{y}-\mathcal{A}\mathcal{D}(\hat{\bm{z}}_{0|t})\big\|_2,
\label{eq:hybrid_dc}
\end{aligned}
\end{equation}
where $\mathcal{D}(\cdot)$ denotes the pre-trained VAE decoder, $\hat{\bm{x}}_{0|t} \coloneqq\mathcal{D}(\hat{\bm{z}}_{0|t})$ (each denote the clean image prediction in pixel and latent space in Eq.~\eqref{eq:tweedie_update}.) and $\mathcal{A}^\dagger$ is the pseudo-inverse of the measurement operator $\mathcal{A}$. The interpolation weight $\omega_t$ modulates the transition from BP to LS and is parameterized as $\omega_t = (1-\sigma_\omega(t))^\phi$ with $\phi=0.8$. In this formulation, we set $\sigma_\omega(t) = \sigma_t$ (noise schedule $\sigma_t$) for flow matching models and $\sigma_\omega(t) = 1 - \bar{\alpha}_t$ (noise variance schedule $1 - \bar{\alpha}_t$) for diffusion models. 

The complete procedure is detailed in Algorithm~\ref{alg:sampler_flow}, with further diffusion-specific sampling variants provided in Algorithm~\ref{alg:sampler_diffusion}.

\begin{algorithm}[!t]
    \caption{Inference of TriPS (Flow Matching)}\label{alg:sampler_flow}
    \begin{algorithmic}[1]
    \Require Measurement $\bm{y}$, linear operator $\mathcal{A}$ (and adjoint $\mathcal{A}^\top$), pre-trained flow model $v_\theta$, VAE encoder/decoder $(\mathcal{E},\mathcal{D})$, text embeddings $(c_\varnothing, c)$, noise schedule $\{\sigma_t\}_{t\in[0,1]}$, Schedules of DC, CFG, Stochasticity scale $\beta(t),\lambda(t),\eta(t)$
    \State $\bm{z}_1 \sim \mathcal{N}(0, I_d)$
    \For{$t: 1\rightarrow 0$}
        \State $v_t(\bm{z}_t) \gets v_\theta(\bm{z}_t, c_\varnothing) + \lambda(t)\big(v_\theta(\bm{z}_t, c) - v_\theta(\bm{z}_t, c_\varnothing)\big)$ \Comment{\textcolor[rgb]{0,0,0}{\textit{1. CFG-induced velocity field}}} 
        \State $\hat{\bm{z}}_{0|t} \gets \bm{z}_t - \sigma_tv_t(\bm{z}_t)$
        \State $\hat{\bm{z}}_{1|t} \gets \bm{z}_t + (1-\sigma_t)v_t(\bm{z}_t)$
        \State $\tilde{\bm{z}}_{0|t}(\bm{y}) \gets \hat{\bm{z}}_{0|t} - \beta(t) \nabla_{\hat{\bm{z}}_{0|t}} \mathcal{L}(\mathcal{A}\mathcal{D}(\hat{\bm{z}}_{0|t}), \bm{y})$ \big(Eq.~\eqref{eq:hybrid_dc}\big)
        \Comment{\textcolor[rgb]{0,0,0}{\textit{2. Data consistency (Sec.~\ref{par:flow_implementation})}}} 
        \State $\bm{\epsilon} \sim \mathcal{N}(0, I_d)$
        \State $\tilde{\bm{z}}_{1|t} \gets \sqrt{1-\eta^2(t)}\bm{z}_{1|t} + \eta(t)\bm{\epsilon}$
        \Comment{\textcolor[rgb]{0,0,0}{\textit{3. Stochasticity}}} 
        \State $\bm{z}_{t+\Delta t} \gets (1-\sigma_{t+\Delta t})\tilde{\bm{z}}_{0|t}(\bm{y}) + \sigma_{t+\Delta t} \tilde{\bm{z}}_{1|t}$
        \Comment{\textcolor[rgb]{0,0,0}{\textit{4. Euler update}}} 
    \EndFor
    \end{algorithmic}
    \label{alg:TriPS_flow}
\end{algorithm}

\begin{algorithm}[!t]

    \caption{Inference of TriPS (Diffusion)}\label{alg:sampler_diffusion}
    \begin{algorithmic}[1]
    \Require Measurement $\bm{y}$, linear operator $\mathcal{A}$ (and adjoint $\mathcal{A}^\top$), pre-trained flow model $v_\theta$, VAE encoder/decoder $(\mathcal{E},\mathcal{D})$, text embeddings $(c_\varnothing, c)$, NFE $T$, noise variance schedule $\{ 1-\bar{\alpha}_t \}_{t\in[T,1]}$, Schedules of DC, CFG, Stochasticity scale $\beta(t),\lambda(t),\eta(t)$
    \State $\bm{z}_1 \sim \mathcal{N}(0, I_d)$
    \For{$t$ from $T$ to 1}
        \State $\bm{\epsilon}_t(\bm{z}_t) \gets \bm{\epsilon}_\theta(\bm{z}_t, c_\varnothing) + \lambda(t)\big(\bm{\epsilon}_\theta(\bm{z}_t, c) - \bm{\epsilon}_\theta(\bm{z}_t, c_\varnothing)\big)$ \Comment{\textcolor[rgb]{0,0,0}{\textit{1. CFG-scaled predicted noise}}} 
        \State $\hat{\bm{z}}_{0|t} \gets \frac{1}{\sqrt{1-\bar{\alpha}_t}}\big(\bm{z}_t - \sqrt{1-\bar{\alpha}_t} \bm{\epsilon}_t(\bm{z}_t)\big)$
        \State $\tilde{\bm{z}}_{0|t}(\bm{y}) \gets \hat{\bm{z}}_{0|t} - \beta(t) \nabla_{\hat{\bm{z}}_{0|t}} \mathcal{L}(\mathcal{A}\mathcal{D}(\hat{\bm{z}}_{0|t}), \bm{y})$ \big(Eq.~\eqref{eq:hybrid_dc}\big)
        \Comment{\textcolor[rgb]{0,0,0}{\textit{2. Data consistency (Sec.~\ref{par:diffusion_implementation})}}} 
        \State $\bm{\epsilon} \sim \mathcal{N}(0, I_d)$
        \State $\tilde{\bm{\epsilon}}_t(\bm{z}_t)  \gets \sqrt{1-\eta^2(t)} \bm{\epsilon}_t(\bm{z}_t) + \eta(t)\bm{\epsilon}$
        \Comment{\textcolor[rgb]{0,0,0}{\textit{3. Stochasticity}}} 
        \State $\bm{z}_{t-1} \gets \sqrt{\bar{\alpha}_{t-1}} \tilde{\bm{z}}_{0|t}(\bm{y}) + \sqrt{1-\bar{\alpha}_{t-1}} \tilde{\bm{\epsilon}}_t(\bm{z}_t)$
        \Comment{\textcolor[rgb]{0,0,0}{\textit{4. Euler update}}} 
    \EndFor
    \end{algorithmic}
    \label{alg:TriPS_diff}
\end{algorithm}
\section{TriPS in Diffusion Models} \label{app:subsec_trips_diffusion}
\subsection{Datasets and Metrics}
For the quantitative evaluation, we utilize 1,000 high-resolution images from the FFHQ dataset. To ensure an objective and consistent comparison with other generative frameworks, the evaluation protocol and performance metrics are identical to those established in Section~\ref{par:data_metric}.

\subsection{Baselines}
To evaluate the performance of TriPS in diffusion-based inverse problems, we employ Stable Diffusion 1.5 as the generative backbone with a fixed resolution of $512 \times 512$. For all experiments, the number of function evaluations (NFE) is set to 50. We compare TriPS against a representative set of diffusion-based baselines, including PSLD~\cite{psld}, DDPG~\cite{ddpg}, P2L~\cite{p2l}, and TReg~\cite{treg}. Detailed implementation specifications for the diffusion setup and baseline configurations are provided in Section~\ref{par:diffusion_implementation}.

\subsection{Problem Setting}
For diffusion-based experiments, we include super-resolution 8$\times$ and motion deblurring with a kernel size of 61 and intensity of 0.5. In all cases, Gaussian noise with a standard deviation of 0.03 (1.5\% of the pixel range) is added.

\subsection{Experimental Results for Diffusion Model}
To evaluate the generalizability of TriPS beyond flow matching frameworks, we extend the proposed method to diffusion models by employing the sampling algorithm formulated in Algorithm~\ref{alg:sampler_diffusion}. We benchmark TriPS on two representative inverse problems: super-resolution $\times8$ and motion deblurring. In both tasks, TriPS demonstrates superior performance compared to established diffusion-based baselines. As summarized in Table~\ref{tab:quanti_diffusion}, the framework achieves a superior distortion-perception trade-off, which suggests that the optimized triadic schedules effectively leverage pre-trained generative priors for image restoration. These quantitative improvements are further supported by the qualitative comparisons in Fig.~\ref{fig:app_quali_diffusion_model}, where TriPS exhibits enhanced visual fidelity and structural consistency. Collectively, these findings validate the robustness and adaptability of the TriPS framework across distinct generative mechanisms.

\begin{figure*}[!ht]
  \centering
  \includegraphics[width=0.9\textwidth]{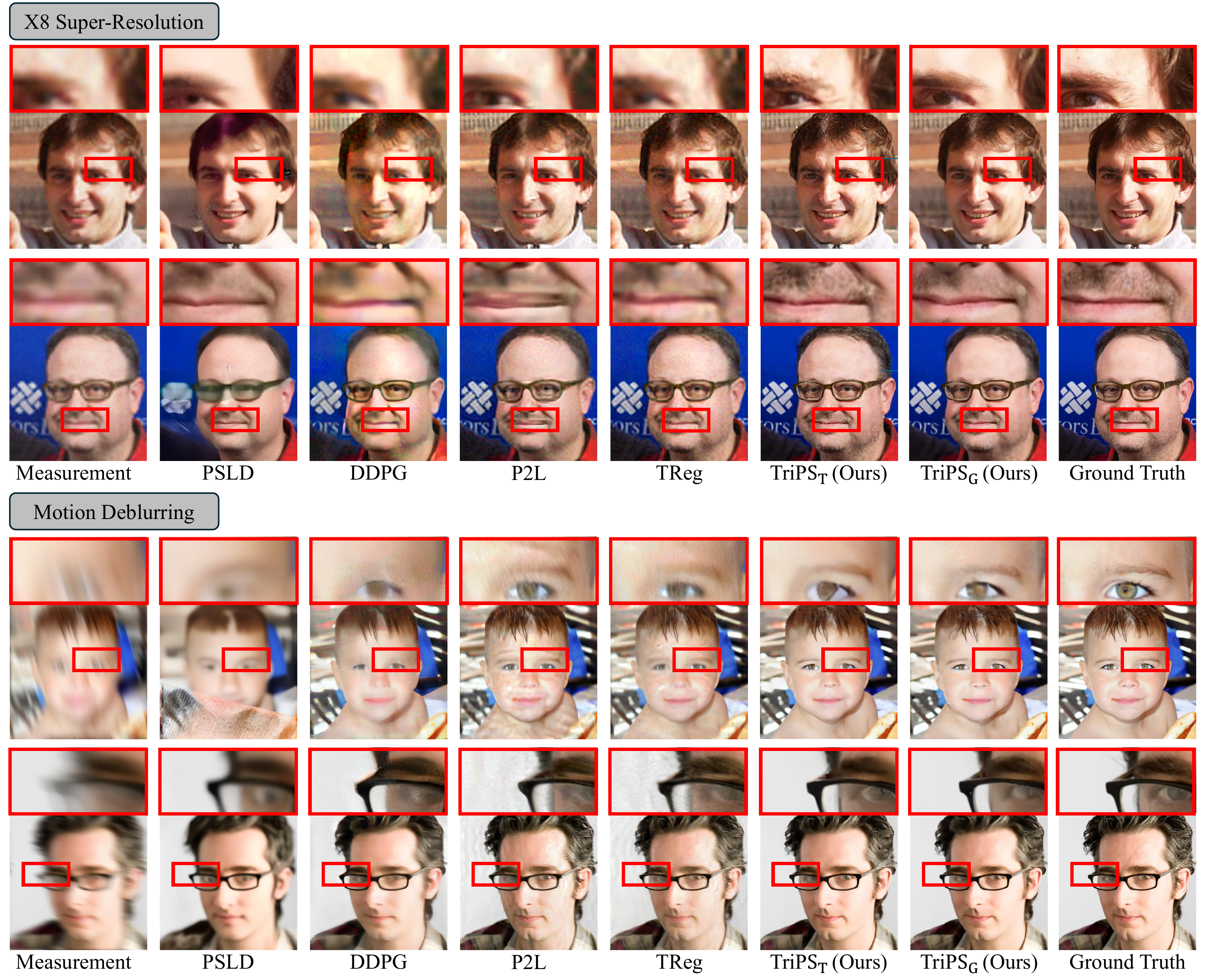}
  \caption{
  Qualitative comparison for FFHQ dataset on linear inverse problems based on Diffusion Model (SD 1.5). 
  }
  \label{fig:app_quali_diffusion_model}
\end{figure*}


\section{Implementation Details}
\label{app:sec_implementation}
\subsection{Implementation Details for Template-based Schedule Search}
\label{app:subsec_imp_detail_template}

\paragraph{Search space}
Each of the three schedule components
$(\beta(t),\,\lambda(t),\,\eta(t))$ is independently assigned
one template from the family
$\mathcal{T}=\{\text{linear},\,\text{exponential},\,
\text{logarithmic}\}$,
yielding a Cartesian product
$\mathcal{T}^{3}$ of $3^{3}=27$ candidate triads
$\tau=(\tau_\beta,\tau_\lambda,\tau_\eta)$.
Monotonicity is enforced by construction: the decreasing
direction is used for $\beta(t)$ and $\eta(t)$, and the
increasing direction for $\lambda(t)$.
 
\paragraph{Template functions}
Given endpoint values $[s_{\min}, s_{\max}]$
and time $t \in [0,1]$ (decreasing from noise to clean),
the three template functions are defined as:
\begin{align*}
  s^{\mathrm{lin}}(t)
    &= s_{\max} - (s_{\max}-s_{\min})\,t, \\[4pt]
  s^{\mathrm{exp}}(t)
    &= s_{\min}
      + (s_{\max}-s_{\min})\,
        \frac{e^{-\gamma t}-e^{-\gamma}}{1-e^{-\gamma}},
      \quad \gamma=3, \\[4pt]
  s^{\mathrm{log}}(t)
    &= s_{\max}
      - (s_{\max}-s_{\min})\,
        \frac{\ln(1+\gamma t)}{\ln(1+\gamma)},
      \quad \gamma=10.
\end{align*}
Linear templates maintain a constant rate of change.
Exponential templates change slowly at first then rapidly
(capturing schedules that front-load the adjustment in the
early/high-noise regime).
Logarithmic templates change rapidly at first then slowly
(suited for schedules that converge quickly and plateau
in the late/low-noise regime).
 
\paragraph{Parameter ranges}
The endpoint values used in the grid search are:
$\lambda(t)\in[1.0,\,6.0]$ and $\eta(t)\in[0,\,1]$,
fixed across all tasks; $\beta(t)\in
[\beta^{\mathrm{T}}_{\min},\,\beta^{\mathrm{T}}_{\max}]$,
which are task-dependent and listed in Tables~5 and~6.
 
\paragraph{Evaluation protocol}
All $27$ candidate triads are evaluated on a calibration
set $\mathcal{D}_{\mathrm{cal}}$ of $100$ images per task,
sampled from indices disjoint from the held-out test set.
For each candidate $\tau$, the utility
$\mathcal{U}(\tau;\mathcal{D}_{\mathrm{cal}})
  = \alpha\cdot\mathrm{PSNR}_{\mathrm{norm}}
  + (1{-}\alpha)\cdot(1{-}\mathrm{LPIPS})$
(with $\alpha=0.5$) (Eq.~\eqref{eq:template_search}) is computed, and the highest-scoring
triad is chosen as $\tau^{\star}$.
Here, $\mathrm{PSNR}_{\mathrm{norm}}$ denotes the min-max normalized PSNR
computed across all candidates in the search,
so that both terms share a common $[0,1]$ scale.

For the implementation of template-based schedule search ($\text{TriPS}_{\text{T}}$), the specific hyperparameter configurations, including $\beta_{\min}^{\text{T}}$ and $\beta_{\max}^{\text{T}}$ as defined in Section~\ref{subsec:template_opt}, are summarized in Table~\ref{tab:temp_id_flow} for flow matching models and Table~\ref{tab:temp_id_diffusion} for diffusion models. These tables further detail the optimized schedule templates identified for each task, providing the necessary parameters to ensure the reproducibility of our experimental results.
\begin{table*}[!ht]
    \centering
    \renewcommand{\arraystretch}{1.1} 
    \setlength{\tabcolsep}{0.1em} 
    \caption{
    The resulting templates of $\text{TriPS}_{\text{T}}$ and hyperparameters for DC guidance scale $[\beta_{\min}^{T}, \beta_{\max}^{T}]$ for linear inverse problems based on the flow matching model.
    }
    \resizebox{1\textwidth}{!}{
    \begin{tabular}{cccc|cccc|cccc|cccc}
    \toprule
    \multicolumn{16}{c}{\textbf{Flow Matching Model (SD3.5-M)}}\\
    \midrule
    \multicolumn{4}{c|}{Super-Resolution $\times$8}
    & \multicolumn{4}{c|}{Super-Resolution $\times$12}
    & \multicolumn{4}{c|}{Motion Deblurring}
    & \multicolumn{4}{c}{Gaussian Deblurring}\\
    \midrule
    
    $\beta(t)\downarrow$ & $\lambda(t)\uparrow$ & $\eta(t)\downarrow$ & $[\beta_{\min}^{T}, \beta_{\max}^{T}]$
    & $\beta(t)\downarrow$ & $\lambda(t)\uparrow$ & $\eta(t)\downarrow$ & $[\beta_{\min}^{T}, \beta_{\max}^{T}]$
    & $\beta(t)\downarrow$ & $\lambda(t)\uparrow$ & $\eta(t)\downarrow$ & $[\beta_{\min}^{T}, \beta_{\max}^{T}]$
    & $\beta(t)\downarrow$ & $\lambda(t)\uparrow$ & $\eta(t)\downarrow$ & $[\beta_{\min}^{T}, \beta_{\max}^{T}]$ \\
    \midrule
    \multicolumn{16}{c}{\textbf{FFHQ (768 $\times$ 768)}}\\
    \midrule
    linear & logarithmic & logarithmic & [50,250]  & linear & logarithmic & logarithmic & [40,250] &  linear & logarithmic & linear & [150,350] & logarithmic & logarithmic & logarithmic & [100,200]  \\
    \midrule
    \multicolumn{16}{c}{\textbf{DIV2K (768 $\times$ 768)}}\\
    \midrule
    linear & logarithmic & logarithmic & [100,300] & linear & logarithmic & logarithmic & [90,250] & linear & exponential & linear & [150,350] & linear & logarithmic & logarithmic & [200,300] \\
    \bottomrule
    \end{tabular}
    }
    \label{tab:temp_id_flow}
\end{table*}

\begin{table*}[!ht]
    \centering
    \renewcommand{\arraystretch}{1.1} 
    \setlength{\tabcolsep}{2.1em} 
    \caption{
    The resulting templates of $\text{TriPS}_{\text{T}}$ and hyperparameters for DC guidance scale $[\beta_{\min}^{T}, \beta_{\max}^{T}]$ for linear inverse problems based on the diffusion model.
    }
    \resizebox{1\textwidth}{!}{
    \begin{tabular}{cccc|cccc}
    \toprule
    \multicolumn{8}{c}{\textbf{Diffusion Model (SD1.5)}}\\
    \midrule
    \multicolumn{4}{c|}{Super-Resolution $\times$8}
    & \multicolumn{4}{c}{Motion Deblurring} \\
    \midrule
    $\beta(t)\downarrow$ & $\lambda(t)\uparrow$ & $\eta(t)\downarrow$ & $[\beta_{\min}^{T}, \beta_{\max}^{T}]$
    & $\beta(t)\downarrow$ & $\lambda(t)\uparrow$ & $\eta(t)\downarrow$ & $[\beta_{\min}^{T}, \beta_{\max}^{T}]$ \\
    \midrule
    linear & exponential & logarithm & [15,30] & exponential & exponential & linear & [30,60] \\
    \bottomrule
    \end{tabular}
    }
    \label{tab:temp_id_diffusion}
\end{table*}

\subsection{Implementations Details for GRPO-based Schedule Optimization}
\label{app:subsec_implement_grpo}
We parameterize the schedules of DC guidance, CFG and stochasticity using the Bernstein-Beta policy, defining a search space bounded by $\lambda_{min}=1, \lambda_{max}=8$, $\beta_{min}=10, \beta_{max}=400$, and $\eta_{min}=0, \eta_{max}=1$, which are specified as $s_{\min}$ and $s_{\max}$ in Sec~\ref{par: schedule_parameterization}. 
We employ a learning rate of $10^{-2}$, group and batch sizes of 4, and a KL divergence coefficient of $10^{-3}$ for a maximum of 200 iterations during optimization.
The task-dependent weights for the reward function, $w_{\text{dist}}$ and $w_{\text{perc}}$, are detailed in Table~\ref{tab:grpo_id_flow} for flow matching models and Table~\ref{tab:grpo_id_diffusion} for diffusion models.
All experiments are conducted on a single NVIDIA A100 GPU using a calibration set of 100 images per task, sampled from indices disjoint from the test set to ensure zero-shot evaluation conditions. 

\paragraph{Mathematical motivation for the Bernstein basis}
The Bernstein basis $B_{k,d}(t)=\binom{d}{k}t^{k}(1-t)^{d-k}$
is selected for two principled reasons.
\textbf{(i)~Partition-of-unity:}
$\sum_{k=0}^{d}B_{k,d}(t)=1$ for all $t\in[0,1]$,
which ensures that the resulting curve
$\tilde{s}(t)=\sum_{k=0}^{d}w^{(s)}_{k}B_{k,d}(t)$
is always a \emph{convex combination} of the sampled
coefficients $\{w^{(s)}_{k}\}$.
\textbf{(ii)~Bounded range:}
since each $w^{(s)}_{k}\sim\mathrm{Beta}(a^{(s)}_{k},b^{(s)}_{k})$
is supported on $(0,1)$ and the basis forms a partition of
unity, $\tilde{s}(t)\in(0,1)$ is guaranteed for \emph{all}
$t$, regardless of the sampled coefficients.
Together, these properties ensure that the affinely
rescaled schedule
$s(t)=s_{\min}+(s_{\max}-s_{\min})\tilde{s}(t)$
remains within $[s_{\min},s_{\max}]$ throughout GRPO
exploration, preventing out-of-range values and stabilizing
policy optimization.

\paragraph{Polynomial degree}
We use degree $d=25$ for all experiments.
This choice provides sufficient expressive capacity to
represent the complex temporal curves identified in
Sec.~5.4 (fine-grained local fluctuations and magnitude
shifts) while keeping the parameter count tractable
($d{+}1=26$ coefficients per schedule component,
$3{\times}26=78$ total policy parameters).

\begin{table*}[!ht]
    \centering
    \renewcommand{\arraystretch}{1.1} 
    \setlength{\tabcolsep}{1.0em} 
    \caption{
    Task-dependent reward weights, $w_{\text{dist}}$ and $w_{\text{perc}}$, utilized during GRPO-based optimization based on flow matching model.
    }
    \resizebox{0.65\textwidth}{!}{
    \begin{tabular}{cc|cc|cc|cc}
    \toprule
    \multicolumn{8}{c}{\textbf{Flow Matching Model (SD3.5-M)}}\\
    \midrule
    \multicolumn{2}{c|}{Super-Resolution $\times$8}
    & \multicolumn{2}{c|}{Super-Resolution $\times$12}
    & \multicolumn{2}{c|}{Motion Deblurring}
    & \multicolumn{2}{c}{Gaussian Deblurring}\\
    \midrule
    
    $w_{\text{dist}}$ & $w_{\text{perc}}$ & $w_{\text{dist}}$ & $w_{\text{perc}}$ & $w_{\text{dist}}$ & $w_{\text{perc}}$ & $w_{\text{dist}}$ & $w_{\text{perc}}$ \\
    \midrule
    \multicolumn{8}{c}{\textbf{FFHQ (768 $\times$ 768)}}\\
    \midrule
    0.3 & 0.7 & 0.5 & 0.5 & 0.5 & 0.5 & 0.3 & 0.7 \\
    \midrule
    \multicolumn{8}{c}{\textbf{DIV2K (768 $\times$ 768)}}\\
    \midrule
    0.3 & 0.7 & 0.4 & 0.6  & 0.3 & 0.7 & 0.3 & 0.7 \\
    \bottomrule
    \end{tabular}
    }
    \label{tab:grpo_id_flow}
\end{table*}

\begin{table*}[!ht]
    \centering
    \renewcommand{\arraystretch}{1.1} 
    \setlength{\tabcolsep}{4.1em} 
    \caption{
    Task-dependent reward weights, $w_{\text{dist}}$ and $w_{\text{perc}}$, utilized during GRPO-based optimization based on diffusion model.
    }
    \resizebox{0.65\textwidth}{!}{
        \begin{tabular}{cc|cc}
        \toprule
        \multicolumn{4}{c}{\textbf{Diffusion Model (SD1.5)}}\\
        \midrule
        \multicolumn{2}{c|}{Super-Resolution $\times$8}
        & \multicolumn{2}{c}{Motion Deblurring} \\
        \midrule
        $w_{\text{dist}}$ & $w_{\text{perc}}$ & $w_{\text{dist}}$ & $w_{\text{perc}}$ \\
        \midrule
        0.3 & 0.7 & 0.5 & 0.5 \\
        \bottomrule
        \end{tabular}
    }
    \label{tab:grpo_id_diffusion}
\end{table*}

\subsection{Implementation Details for Baseline Hyperparameter Settings}
\label{app:subsec_baseline_hyperpar}
\paragraph{Flow matching model baselines}
\label{par:flow_implementation}
For all flow matching evaluations, we utilize Stable Diffusion 3.5-Medium as the generative backbone with the time scheduler shift factor set to 4.0 and the number of function evaluations (NFE) fixed at 28. The specific configurations for each method are as follows:
\begin{itemize}
    \item \textbf{Resample} We adopt the resampling hyperparameter $\gamma \left(\frac{1-\bar\alpha_{t-1}}{\bar\alpha_t}\right)\left(1-\frac{\bar\alpha_t}{\bar\alpha_{t-1}}\right)$ with $\gamma=40$ as proposed in the original paper. To enforce hard data consistency, the skip step size is set to 1, accompanied by 20 steps of gradient descent with a step size (DC guidance scale) of 30. The CFG scale is fixed at 2.0.
    \item \textbf{FlowChef} We employ 3 steps of gradient descent with a step size (DC guidance scale) of 1 for data consistency optimization. The CFG scale is fixed at 2.0.
    \item \textbf{FlowDPS} We employ 3 steps of gradient descent with a step size (DC guidance scale) of 15 for data consistency optimization. The CFG scale is fixed at 2.0.
    \item \textbf{FLAIR} We employ 15 steps of gradient descent for data consistency optimization. The learning rates are set to 2 for super-resolution $12\times$ and 0.1 for deblurring tasks, consistent with the original paper. For the SR $\times8$ task, which was not specified in the original work, we determined the optimal learning rate to be 6 via a grid search on the calibration set. The CFG scale is fixed at 2.0.
    \item \textbf{TriPS} TriPS employs time-varying schedules for the DC guidance scale $\beta(t)$, CFG scale $\lambda(t)$, and stochasticity $\eta(t)$, as derived from our triadic schedule optimization. For the data consistency update, we utilize $N=6$ gradient descent steps. The inner DC step size $\beta_{dc}$ is defined as $\beta_{dc} = \beta(t) \cdot (0.25 + 0.75 \sigma_t^2) / N$, where $\sigma_t$ represents the noise schedule.
\end{itemize}

\paragraph{Diffusion model baselines} \label{par:diffusion_implementation}
For all diffusion-based evaluations, we utilize Stable Diffusion 1.5 as the generative backbone with the NFE fixed at 50. To ensure a competitive comparison, hyperparameters for all baselines are determined via an extensive grid search on a calibration set. Specific configurations for each method are as follows:
\begin{itemize}
    \item \textbf{PSLD} 
    We set the DC and gluing update step sizes to 1.0 for super-resolution (SR). For motion deblurring, the DC step size is adjusted to 10, while the gluing update step size remains 1.0.
    \item \textbf{DDPG} We set hyperparameters for SR and motion deblurring set to $\{\gamma, \zeta, \tilde{\eta}\} = \{10.0, 0.8, 0.3\}$ and $\{5.0, 0.6, 0.6\}$, respectively, with the guidance step size $\mu_t$ set to the theoretically safe value $\mu_t^*$. Here, $\gamma$ controls the guidance transition from back-projection to least-squares, $\zeta$ balances stochastic noise injection against reconstruction accuracy, and $\tilde{\eta}$ regularizes the back-projection operator to prevent noise amplification.
    \item \textbf{P2L} We employ $K=1$ text embedding update per timestep with a learning rate of $10^{-4}$ and a regularization weight $\lambda=10^{-6}$. The projection interval $\gamma$ and DC guidance scale $\rho_t$ are set to $\{20, 0.05\}$ for SR and $\{10, 0.05\}$ for motion deblurring, respectively.
    \item \textbf{TReg} We employ a fixed CFG scale of 2.0. The CG update range is restricted to $t \le 50$ for every third timestep ($t \pmod 3 = 0$), without utilizing DPS updates. Other parameters follow the original paper.
    \item \textbf{TriPS} TriPS employs time-varying schedules for the DC guidance scale $\beta(t)$, CFG scale $\lambda(t)$, and stochasticity $\eta(t)$, as derived from our triadic schedule optimization. For the data consistency update, we utilize $N=6$ gradient descent steps. The inner DC step size $\beta_{dc}$ is defined as $\beta_{dc} = \beta(t) \cdot (0.25 + 0.75 (1-\bar{\alpha}_t)^2) / N$, where $1-\bar{\alpha}_t$ represents the noise variance schedule.
\end{itemize}

\section{Additional Experiments}
\label{app:sec_add_exp}

\subsection{Additional Task: Inpainting}
\label{app:subsec_discuss_inpainting}
\paragraph{Triadic scheduling trend on inpainting}
\label{par:trend_inpainting}
While our primary framework suggests a tapering DC guidance scale $\beta(t)$ to prevent the amplification of measurement noise in the late sampling stages, we identify inpainting as a notable exception that necessitates a distinct scheduling behavior. Unlike blur or sub-sampling kernels that act globally, the inpainting operator $\mathbf{A}$ provides exact, high-frequency likelihood supervision for the unmasked regions. In this context, the late-stage increase of $\beta(t)$ does not lead to noise injection but rather serves as a crucial anchoring mechanism.

Our time-window ablation in Table~\ref{tab:ablation_inpainting_schedule_shape} reveals that a monotonic increasing $\beta(t)$ schedule outperforms the standard tapering strategy across both fidelity (PSNR, SSIM) and perceptual metrics (LPIPS, patch-based KID). Specifically, we evaluate on 100 FFHQ images within the range $[\beta_{\min}^{\text{T}}, \beta_{\max}^{\text{T}}] = [80, 240]$: Fixed (constant at the mean), Linearly (increasing/decreasing), and non-monotonic  Tent and V-shape function~\cite{smith2017cyclical} curves (linearly increasing to decreasing / linearly decreasing to increasing) for inpainting task. This divergence stems from the spatial nature of the inpainting task: as the sampling process approaches the low-noise regime, maintaining or even intensifying the DC guidance ensures a seamless transition between the generated content and the ground-truth pixels at the mask boundaries. Without this late-stage reinforcement, the generative manifold may slightly drift away from the hard constraints of the known pixels, resulting in boundary artifacts. Consequently, for tasks with strong spatial-domain likelihood, we refine our triadic strategy to favor a late-stage boost in $\beta(t)$, which effectively keeps the high-frequency details into the measurement-consistent subspace.

\begin{table}[!ht]
    \centering
    \renewcommand{\arraystretch}{1.1} 
    \setlength{\tabcolsep}{2.0em}    

    \caption{Time-window ablation on $\beta(t)$ scheduling on inpainting task (FFHQ).}
    \resizebox{0.7\textwidth}{!}{
    \begin{tabular}{l|cccc}
    \toprule
    Scheduling Pattern & PSNR$\uparrow$ & SSIM$\uparrow$ & KID$\downarrow$ & LPIPS$\downarrow$ \\
    \midrule
    Fixed ($\rightarrow$)
    & 22.93 & 0.835 & 0.008 & 0.094 \\
    Linearly decreasing ($\searrow$)
    & 22.92 & 0.830 & 0.009 & 0.101 \\
    Tent function ($\nearrow\!\searrow$)
    & \underline{23.18} & 0.836 & 0.008 & 0.097 \\
    V-shape function ($\searrow\!\nearrow$)
    & 23.07 & \underline{0.842} & \underline{0.007} & \underline{0.092} \\
    \midrule
    Linearly increasing ($\nearrow$)
    & \textbf{23.23} & \textbf{0.847} & \textbf{0.006} & \textbf{0.089} \\
    \bottomrule
    \end{tabular}
    }
    \label{tab:ablation_inpainting_schedule_shape}
\end{table}

\paragraph{Experimental results for inpainting}
Beyond the four primary inverse problems discussed in Sec.~\ref{subsec:main_exp_results} in the main text, we further evaluate the generalizability of our framework on the inpainting task. As summarized in Table~\ref{tab:quanti_inpainting_flow} and visualized in Fig.~\ref{fig:app_quali_inpainting_flow}, $\text{TriPS}_{\text{T}}$ consistently achieves state-of-the-art performance across box mask inpainting scenario. By leveraging the optimized triadic schedules, specifically the monotonic increasing $\beta(t)$ tailored for spatial consistency, our method effectively resolves the boundary artifacts and blurry textures often observed in baseline posterior sampling methods. Notably, $\text{TriPS}_{\text{T}}$ maintains a superior balance in the perception-distortion trade-off, yielding the highest fidelity (PSNR, SSIM) while simultaneously minimizing perceptual distance (LPIPS, patch-based KID).

\begin{table*}[!ht]
    \centering
    \renewcommand{\arraystretch}{0.9}
    \setlength{\tabcolsep}{0.6em}
    \caption{Quantitative comparison on box-inpainting task based on flow matching model. Evalutation metrics are averaged over 1,000 samples from the FFHQ dataset and 800 samples from the DIV2K dataset each using 28 NFEs under additive Gaussian noise ($\sigma_n = 0.03$). For each metric, the best and second-best results are indicated in \textbf{bold} and \underline{underline}, respectively.}
    \label{tab:quanti_inpainting_flow}
    \resizebox{0.7\textwidth}{!}{
    \begin{tabular}{l|cccc|cccc}
    \toprule
    \multicolumn{9}{c}{\textbf{Flow Matching Model (SD3.5-M)}}\\
    \midrule
    & \multicolumn{4}{c|}{\textbf{FFHQ (768 $\times$ 768)}}
    & \multicolumn{4}{c}{\textbf{DIV2K (768 $\times$ 768)}}\\
    \midrule
    Method
    & PSNR$\uparrow$ & SSIM$\uparrow$ & FID$\downarrow$ & LPIPS$\downarrow$
    & PSNR$\uparrow$ & SSIM$\uparrow$ & FID$\downarrow$ & LPIPS$\downarrow$ \\
    \midrule
    ReSample
    & 17.45 & 0.761 & 63.20 & 0.208
    & 19.93 & 0.729 & 31.65 & 0.159 \\
    FlowChef
    & 17.96 & 0.714 & 50.85 & 0.239
    & 19.50 & 0.577 & 51.21 & 0.265 \\
    FlowDPS
    & 17.01 & 0.722 & 38.32 & 0.238
    & 20.06 & 0.633 & 25.70 & 0.202 \\
    FLAIR
    & \underline{22.63} & \underline{0.832} & \textbf{10.29} & \underline{0.095}
    & \underline{23.68} & \underline{0.823} & \underline{11.34} & \underline{0.069} \\
    \midrule
    \rowcolor{red!10}
    $\text{TriPS}_{\text{T}}$ (Ours)
    & \textbf{23.30} & \textbf{0.845} & \underline{12.85} & \textbf{0.090}
    & \textbf{24.27} & \textbf{0.835} & \textbf{9.03} & \textbf{0.068} \\
    \bottomrule
    \end{tabular}
    }
\end{table*}

\begin{figure*}[!ht]
  \centering
  \includegraphics[width=0.95\textwidth]{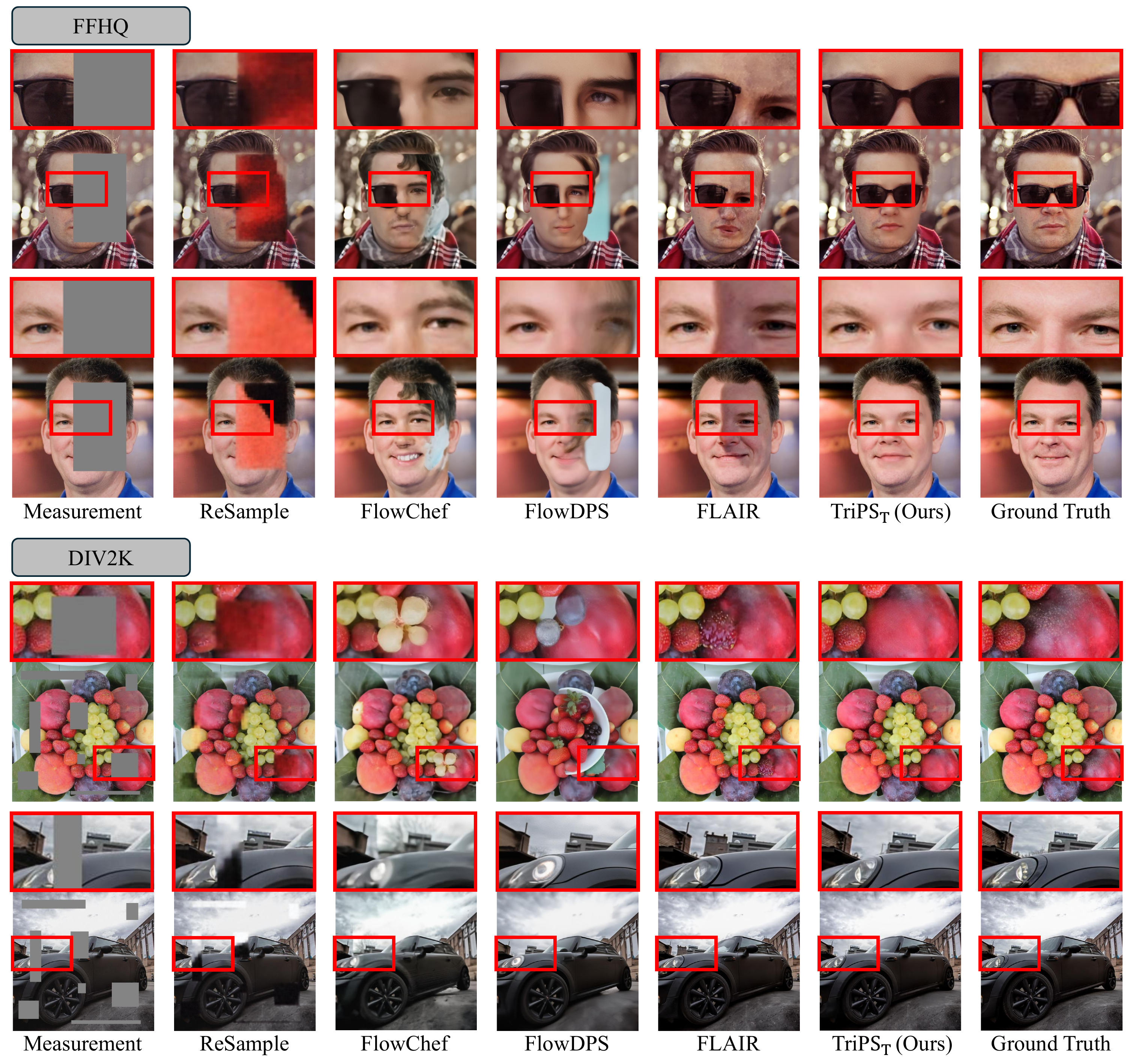}
  \caption{
  Qualitative comparison for FFHQ and DIV2K datasets on box-inpainting task. 
  }
  \label{fig:app_quali_inpainting_flow}
\end{figure*}

\subsection{Applicability of Triadic Schedule Optimization across Different Backbone Solvers}
\label{app:subsec_applicability_trips}
To evaluate the versatility of the TriPS framework, we extend its application to flow matching-based solvers with distinct algorithmic foundations: FlowDPS~\cite{flowdps}, which employs posterior sampling, and FLAIR~\cite{flair}, which utilizes variational inference. Following the $\text{TriPS}_{\text{T}}$ described in Sec.~\ref{subsec:template_opt}, we derive task-specific schedules for DC, CFG, and stochasticity scales for an super-resolution $\times8$ on the FFHQ dataset. As illustrated in Fig.~\ref{fig:TriPS_general_sampler}, $\text{TriPS}_{\text{T}}$ enhances perceptual quality compared to default baselines by restoring high-frequency details while maintaining structural fidelity. 
Quantitative results in Table~\ref{tab:quanti_backbone_versatility_ffhq_small} further confirm that $\text{TriPS}_{\text{T}}$ consistently improves perceptual scores (KID, LPIPS) while maintaining competitive performance in distortion metrics (PSNR, SSIM) relative to the original solvers. These findings demonstrate that our proposed triadic schedule optimization is compatible with diverse inverse problem solvers.

\begin{figure}[t!]
  \centering
  \includegraphics[width=0.7\textwidth]{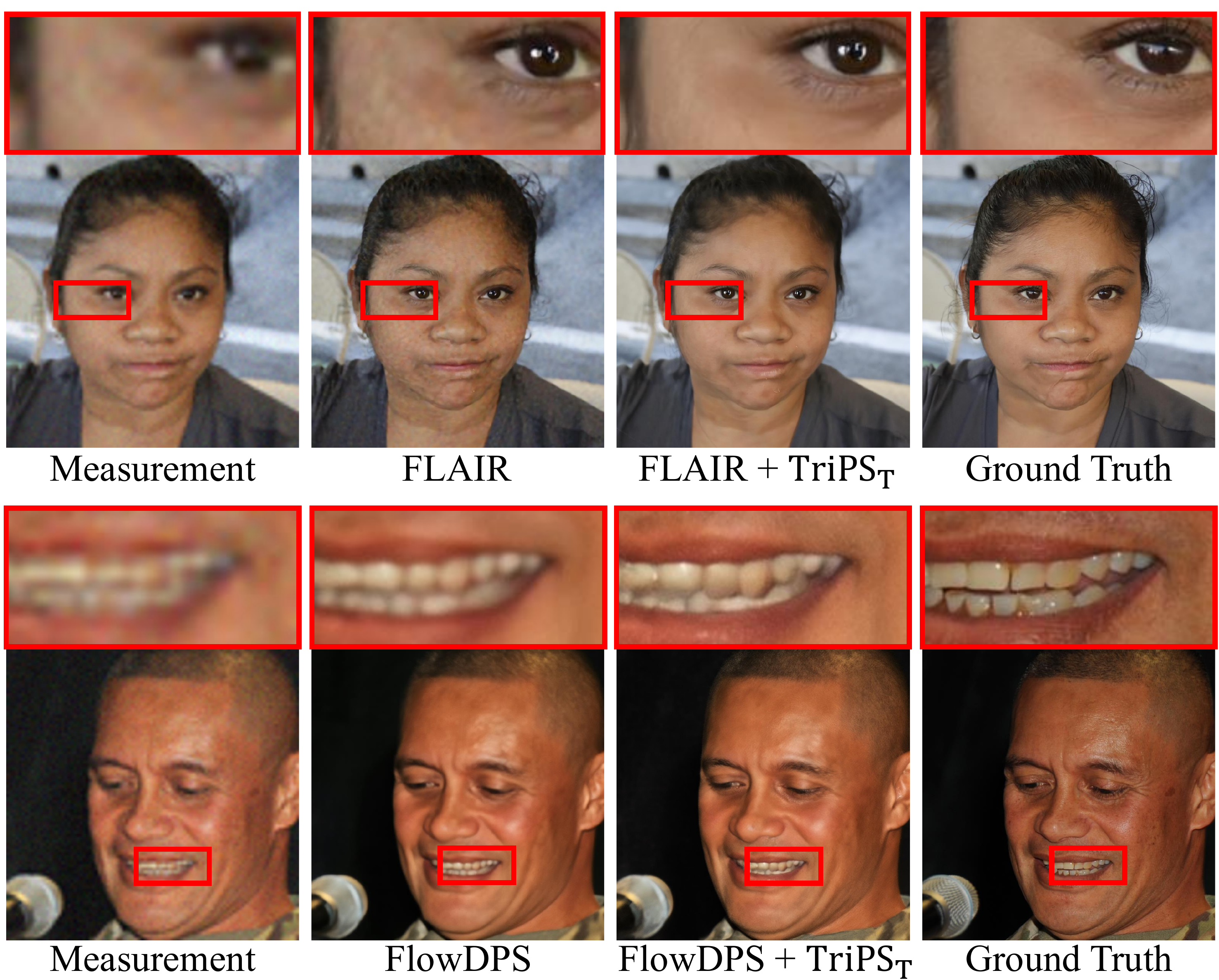}
  \caption{
  Applicability of triadic schedule optimization across distinct flow matching based algorithms. Qualitative results of FlowDPS (posterior sampling) and FLAIR (variational inference) on a Super-Resolution $\times8$ using the FFHQ dataset. Triadic schedules obtained via $\text{TriPS}_{\text{T}}$ enhance high-frequency textures and structural fidelity relative to default baselines.
  }
  \label{fig:TriPS_general_sampler}
\end{figure}

\begin{table}[t!]
    \centering
    \setlength{\tabcolsep}{1.6em}
    
    \caption{
    Quantitative comparison of triadic schedule optimization across diverse backbone solvers. We evaluate the performance on an super-resolution $\times8$ using 100 images from the FFHQ dataset. The results compare the original FlowDPS~\cite{flowdps} and FLAIR~\cite{flair} baselines against their counterparts integrated with our proposed triadic schedule optimization framework.
    }
    \label{tab:quanti_backbone_versatility_ffhq_small}
    \begin{tabular}{l|cccc}
    \toprule
    Method & PSNR$\uparrow$ & SSIM$\uparrow$ & KID$\downarrow$ & LPIPS$\downarrow$ \\
    \midrule
    FlowDPS
    & \textbf{28.09} & 0.779 & 0.012 & 0.105 \\
    \rowcolor{red!10}
    FlowDPS + $\text{TriPS}_{\text{T}}$
    & 27.97 & \textbf{0.782} & \textbf{0.009} & \textbf{0.099} \\
    \midrule
    FLAIR
    & \textbf{29.18} & 0.778 & 0.041 & 0.120 \\
    \rowcolor{red!10}
    FLAIR + $\text{TriPS}_{\text{T}}$
    & 29.03 & \textbf{0.784} & \textbf{0.008} & \textbf{0.099} \\
    \bottomrule
    \end{tabular}
    
\end{table}

\subsection{Additional Ablation: Triadic Schedule Optimization Strategies}
\label{app:subsec_ablation_triadic_schedule}
We investigate the template-based schedule search and GRPO-based schedule optimization frameworks against a Zeroth-Order (ZO) baseline for a Gaussian deblurring task using a flow matching-based solver on the FFHQ dataset. The ZO baseline utilizes an ES-style Gaussian perturbation gradient estimator with Adam updates.
In this setup, the ZO schedule adopts the same Beta-Bernstein parameterization described in Sec.~\ref{par: schedule_parameterization} and is initialized using the schedule obtained from the template-based search. 
As illustrated in Fig.~\ref{fig:grpo_results_curve}, the optimized schedule curve for ZO remains closely aligned with the initial template-based schedule, whereas the GRPO-derived curve exhibits a distinct departure from its initialization. These visual differences indicate that the ZO method suffers from poor convergence in the high-dimensional Beta-Bernstein parameter space, failing to effectively optimize the target objective compared to the GRPO-based approach.

\begin{figure}[!ht]
  \centering
  \includegraphics[width=0.8\linewidth]{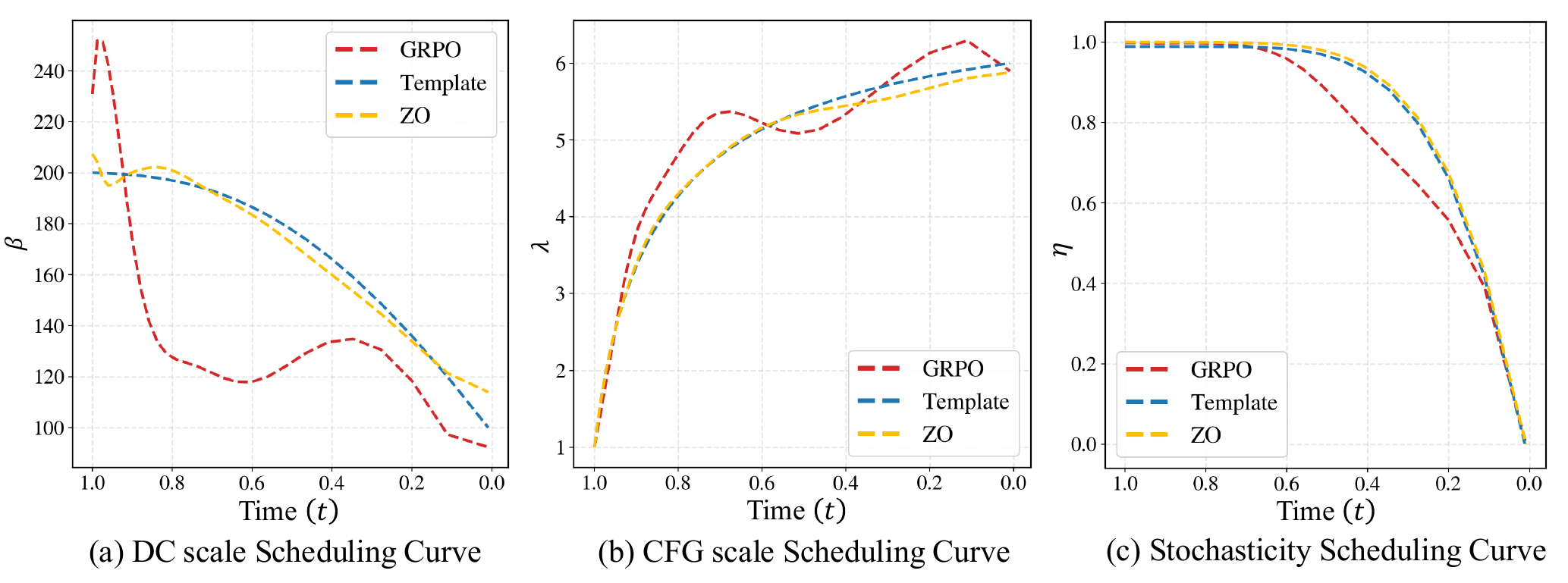}
  \caption{
  Visualization of optimized triadic schedule curves. Comparison of the schedules for DC, CFG, and Stochasticity scales obtained via template-based search, GRPO-based optimization, and the Zeroth-Order (ZO) baseline. The ZO-optimized curves remain closely aligned with the initial template-based schedules, suggesting suboptimal convergence in the parameter space, whereas the GRPO-based schedules demonstrate distinct trajectories specifically tailored to the restoration objective.
  }
  \label{fig:grpo_results_curve}
\end{figure}

\subsection{Additional Ablation: More Results for Impact of Joint Triadic Schedule Optimization}
\label{app:subsec_ablation_more_joint}
To further examine the interplay among the triadic components during optimization, we extend the ablation study in Sec.~\ref{subsec:ablation_studies} by evaluating configurations where only two of the three components are jointly optimized. In these settings, the non-optimized component is fixed to the midpoint of its predefined parameter range(See Sec.~\ref{app:subsec_imp_detail_template}), while the remaining two components are obtained via the $\text{TriPS}_{\text{T}}$ framework. As summarized in Table~\ref{tab:ablation_studies_joint_app}, the joint optimization of all three components yields superior performance across both distortion and perceptual metrics compared to any dual-component configuration. These findings underscore that the synergistic interaction of the triadic schedules is essential for effectively navigating the perception-distortion trade-off, thereby justifying the necessity of the proposed unified search space.

\begin{table}[!ht]
    \centering
    \renewcommand{\arraystretch}{0.9} 
    \setlength{\tabcolsep}{1.6em} 
    \caption{
    Additional ablation study for the impact of joint triadic schedule optimization of TriPS. Super-resolution $\times8$ results evaluated on 100 FFHQ images.
    }
    \resizebox{0.7\textwidth}{!}{
    
        \begin{tabular}{ccc|cccc}
        \toprule
        DC & CFG & Stoch. &
        PSNR$\uparrow$ & SSIM$\uparrow$ & KID$\downarrow$ & LPIPS$\downarrow$ \\
        \midrule
        \cmark & \cmark & \xmark & 29.41 & 0.807 & 0.116 & 0.126 \\
        \xmark & \cmark & \cmark & 29.53 & 0.804 & 0.120 & 0.133 \\
        \cmark & \xmark & \cmark & \underline{29.76} & \underline{0.819} & 0.073 & \underline{0.104} \\
        \rowcolor{red!10}
        \cmark & \cmark & \cmark & \textbf{29.77} & \textbf{0.820} & \textbf{0.072} & \textbf{0.101} \\
        \bottomrule
        \end{tabular}
    }
    \label{tab:ablation_studies_joint_app}
\end{table}

\subsection{Additional Ablation: Optimized Schedule Curve Analysis}
\label{app:subsec_ablation_more_discussion_PD_control}
In this section, we provide an extended analysis of the reward-guided control mechanisms governing the perception–distortion trade-off, supplementing the results presented in Sec.~\ref{subsec:ablation_studies}. The optimized schedules in Fig.~\ref{fig:grpo_results_overall}(c) reveal that perception-oriented posterior sampling favors high initial DC guidance followed by a rapid late-stage decay, whereas distortion-oriented posterior sampling maintains overall higher DC guidance characterized by an increasing-then-decreasing profile. Furthermore, perception-oriented sampling consistently employs larger CFG scales throughout the trajectory to leverage semantic guidance for high-frequency detail synthesis. This pattern aligns with established findings in text-to-image generation~\cite{wang2024analysis}, where such CFG scheduling is shown to prioritize perceptual fidelity over pixel-wise reconstruction. Finally, the elevated stochasticity in the perception-oriented variant, relative to its distortion-oriented counterpart, indicates that the optimization actively exploits the regularizing effects identified in Sec.~\ref{subsec:analysis_sto_regularize} to mitigate distributional mismatches with the ground truth.

\subsection{Runtime Comparison}
\label{app:subsec_runtime_anal}
To evaluate the inference efficiency of TriPS, we compare its average runtime with various diffusion and flow-based baselines. Table~\ref{tab:efficiency_comparison} summarizes these results, reporting the average sampling time per image (s/image) averaged over 10 samples.

\begin{table*}[!ht]
    \centering
    \renewcommand{\arraystretch}{1.0} 
    \setlength{\tabcolsep}{0.8em}    
    \caption{Comparison of different methods in terms of runtime based on flow matching Model and diffusion model. We report the sampling runtime (seconds per image) averaged over 10 samples.}

    \begin{subtable}{\textwidth}
        \centering
        \resizebox{0.70\textwidth}{!}{
        \begin{tabular}{l|c|c|c|c}
        \toprule
        \multicolumn{5}{c}{\textbf{Flow Matching Model (SD3.5-M)}} \\ \midrule
        \multirow{2}{*}{\textbf{Method}} & Super-Resolution $\times$8 & Super-Resolution $\times$12 & Motion Deblurring & Gaussian Deblurring \\ \cmidrule(lr){2-2} \cmidrule(lr){3-3} \cmidrule(lr){4-4} \cmidrule(lr){5-5}
         & Runtime (s) $\downarrow$ & Runtime (s) $\downarrow$ & Runtime (s) $\downarrow$ & Runtime (s) $\downarrow$  \\ \midrule
        ReSample & 11.31 & 11.29 & 12.42 & 10.91 \\
        FlowChef & 4.68 & 4.64 & 5.36 & 4.94 \\
        FlowDPS  & 4.71 & 4.67 & 5.42 & 4.97 \\
        FLAIR    & 11.48 & 11.46 & 12.90 & 11.99 \\
        \textbf{TriPS (Ours)} & 6.58 & 6.48 & 9.87 & 7.45  \\ \bottomrule
        \end{tabular}
        }
    \end{subtable}

    \vspace{1.0em}

    \begin{subtable}{\textwidth}
        \centering
        \resizebox{0.38\textwidth}{!}{
        \begin{tabular}{l|c|c}
        \toprule
        \multicolumn{3}{c}{\textbf{Diffusion Model (SD1.5)}} \\ \midrule
        \multirow{2}{*}{\textbf{Method}} & Super-Resolution $\times$8 & Motion Deblurring \\ \cmidrule(lr){2-2} \cmidrule(lr){3-3}
         & Runtime (s) $\downarrow$  & Runtime (s) $\downarrow$  \\ \midrule
        PSLD & 6.09 & 6.73  \\
        DDPG & 2.15 & 2.58  \\
        P2L  & 9.18 & 9.79  \\
        TReg & 2.73 & 3.10 \\
        \textbf{TriPS (Ours)} & 6.01  & 9.52  \\ \bottomrule
        \end{tabular}
        }
    \end{subtable}
    \vspace{0.5em}
    \label{tab:efficiency_comparison}
\end{table*}

\section{Additional Qualitative Results}
We present additional qualitative results for motion deblurring, super-resolution $\times$12, and Gaussian deblurring. Across all tasks, $\text{TriPS}_{\text{T}}$ and $\text{TriPS}_{\text{G}}$ consistently achieve superior fidelity-perception trade-offs compared to existing baselines, producing more realistic and structurally faithful reconstructions. ReSample and FlowChef tend to produce over-smoothed reconstructions, losing fine-grained facial details such as hair strands, skin textures, and eye contours. FlowDPS and FLAIR recover sharper outputs in some cases, but occasionally at the cost of introducing reconstruction artifacts or hallucinated features that deviate from the ground truth. In contrast, both $\text{TriPS}_{\text{T}}$ and $\text{TriPS}_{\text{G}}$ faithfully recover high-frequency details while maintaining structural consistency with the original image. For the super-resolution$\times$12 setting, most baselines struggle to resolve fine structures from severely downsampled measurements, whereas $\text{TriPS}_{\text{T}}$ and $\text{TriPS}_{\text{G}}$ consistently produce plausible high-frequency details that closely match the ground truth. Similarly, in motion deblurring and Gaussian deblurring, our methods recover sharper edges and more natural textures compared to the baselines, particularly in regions with complex patterns. $\text{TriPS}_{\text{G}}$ further navigates the perception-distortion trade-off via GRPO-optimized triadic schedules, enabling flexible control over the balance between perceptual quality and measurement fidelity. These results corroborate the quantitative improvements reported in Tables~\ref{tab:quanti_flow} and~\ref{tab:transfer_two_tasks_compact}, confirming that triadic schedule optimization translates to visible gains in reconstruction quality across diverse degradation settings. Additional comparisons on the DIV2K dataset further validate the generalization capability of our approach beyond face-centric images.

\begin{figure*}[tp]
  \centering
  \includegraphics[width=0.83\textwidth]{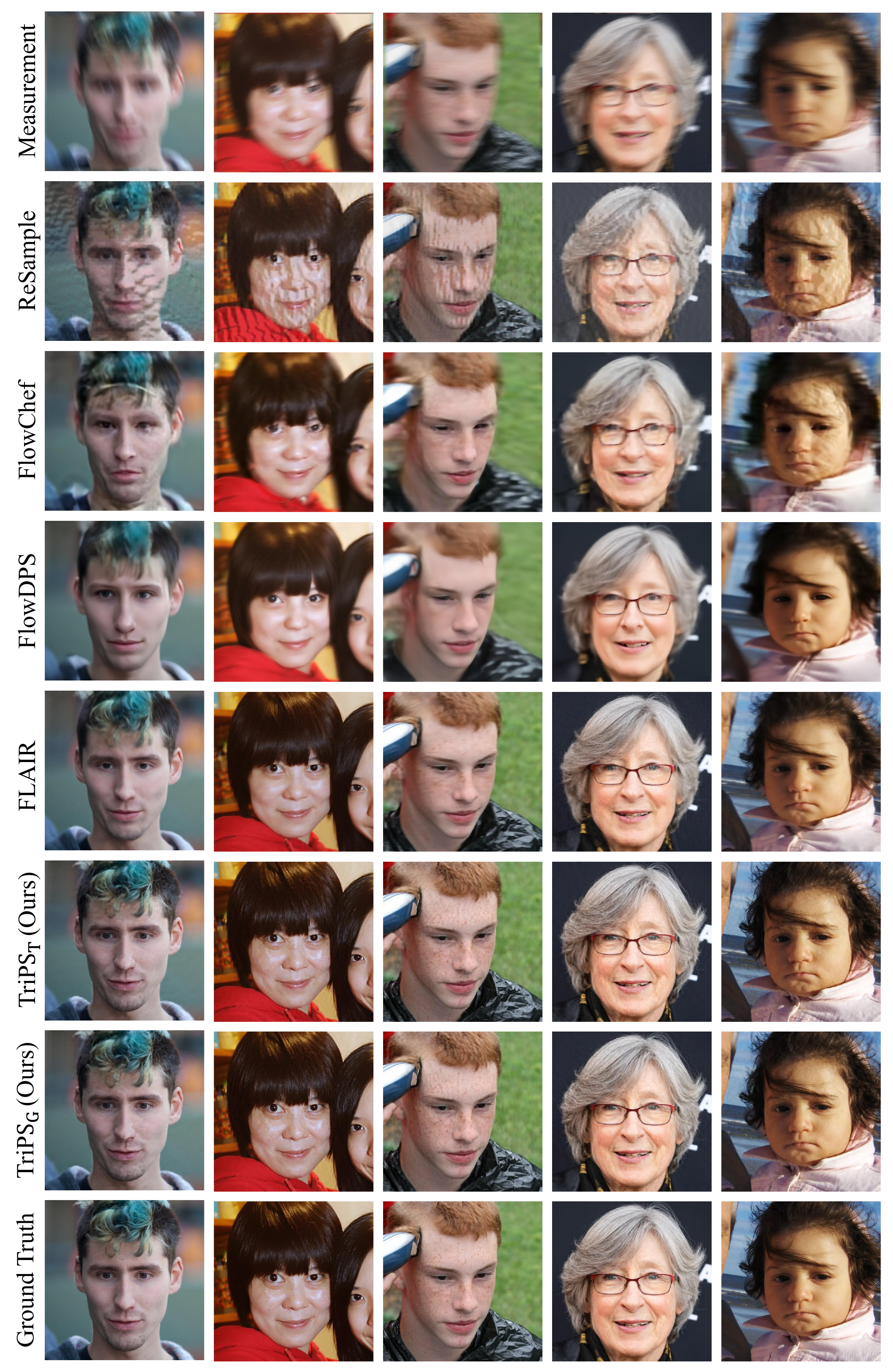}
  \caption{
  Additional qualitative comparison for the motion deblurring on the FFHQ dataset.
  }
  \label{fig:Add_app_quali_MD_FFHQ}
\end{figure*}

\newpage

\begin{figure*}[tp]
  \centering
  \includegraphics[width=0.83\textwidth]{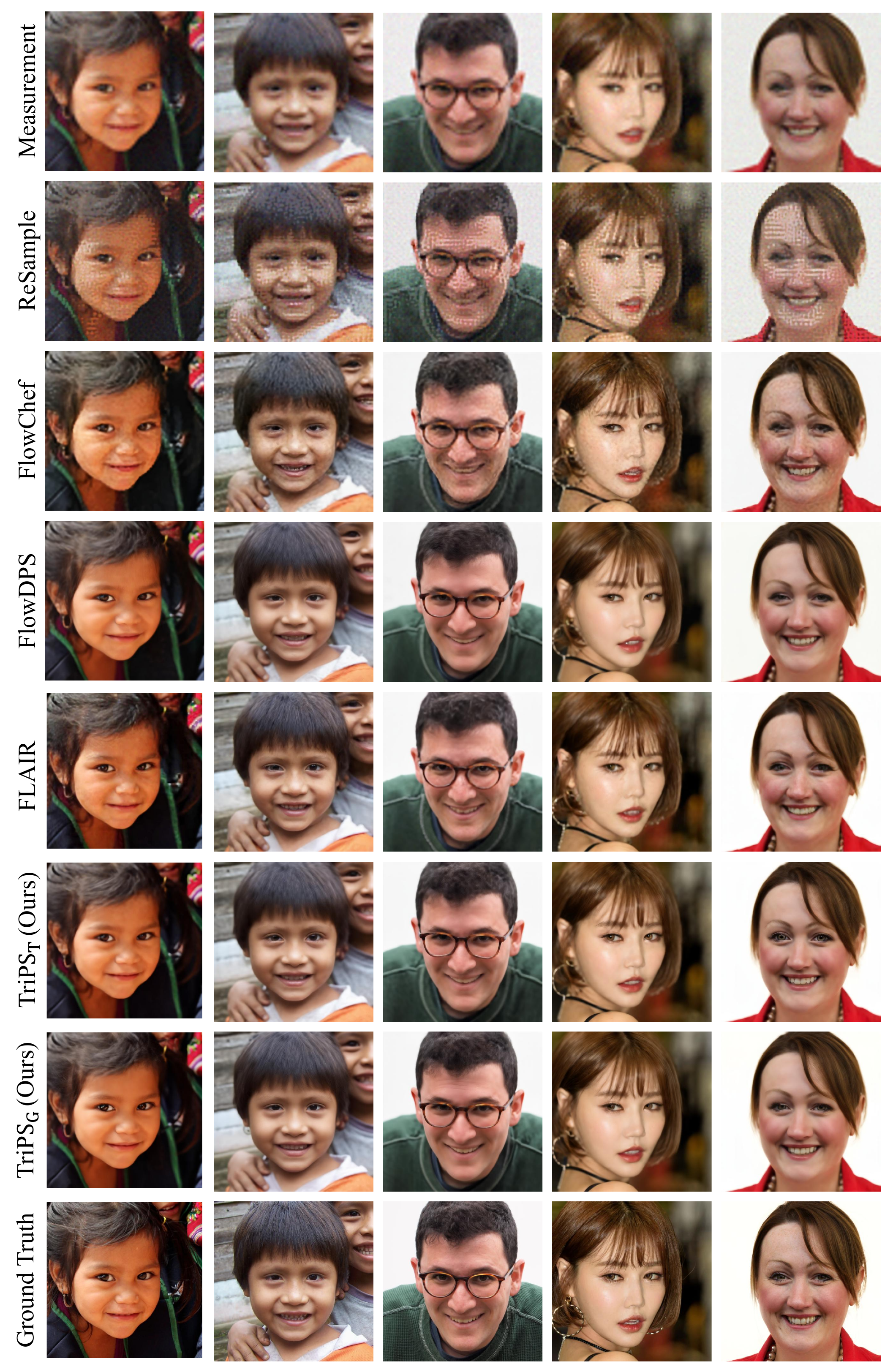}
  \caption{
  Additional qualitative comparison for the super-resolution $\times$12 on the FFHQ dataset.
  }
  \label{fig:Add_app_quali_SRx12_FFHQ}
\end{figure*}

\newpage

\begin{figure*}[tp]
  \centering
  \includegraphics[width=0.83\textwidth]{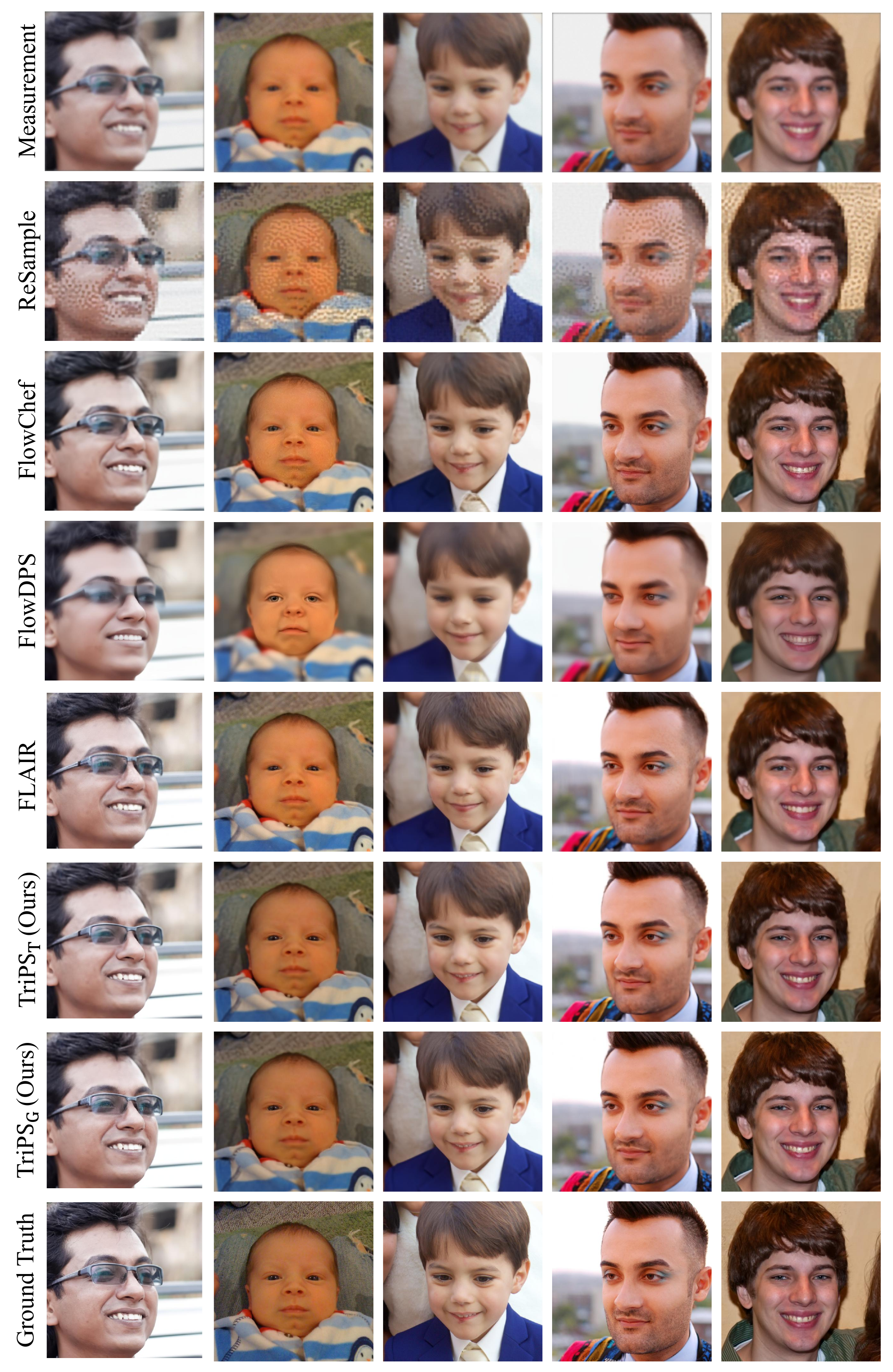}
  \caption{
  Additional qualitative comparison for the gaussian deblurring on the FFHQ dataset.
  }
  \label{fig:Add_app_quali_GD_FFHQ}
\end{figure*}

\newpage

\begin{figure*}[tp]
  \centering
  \includegraphics[width=0.83\textwidth]{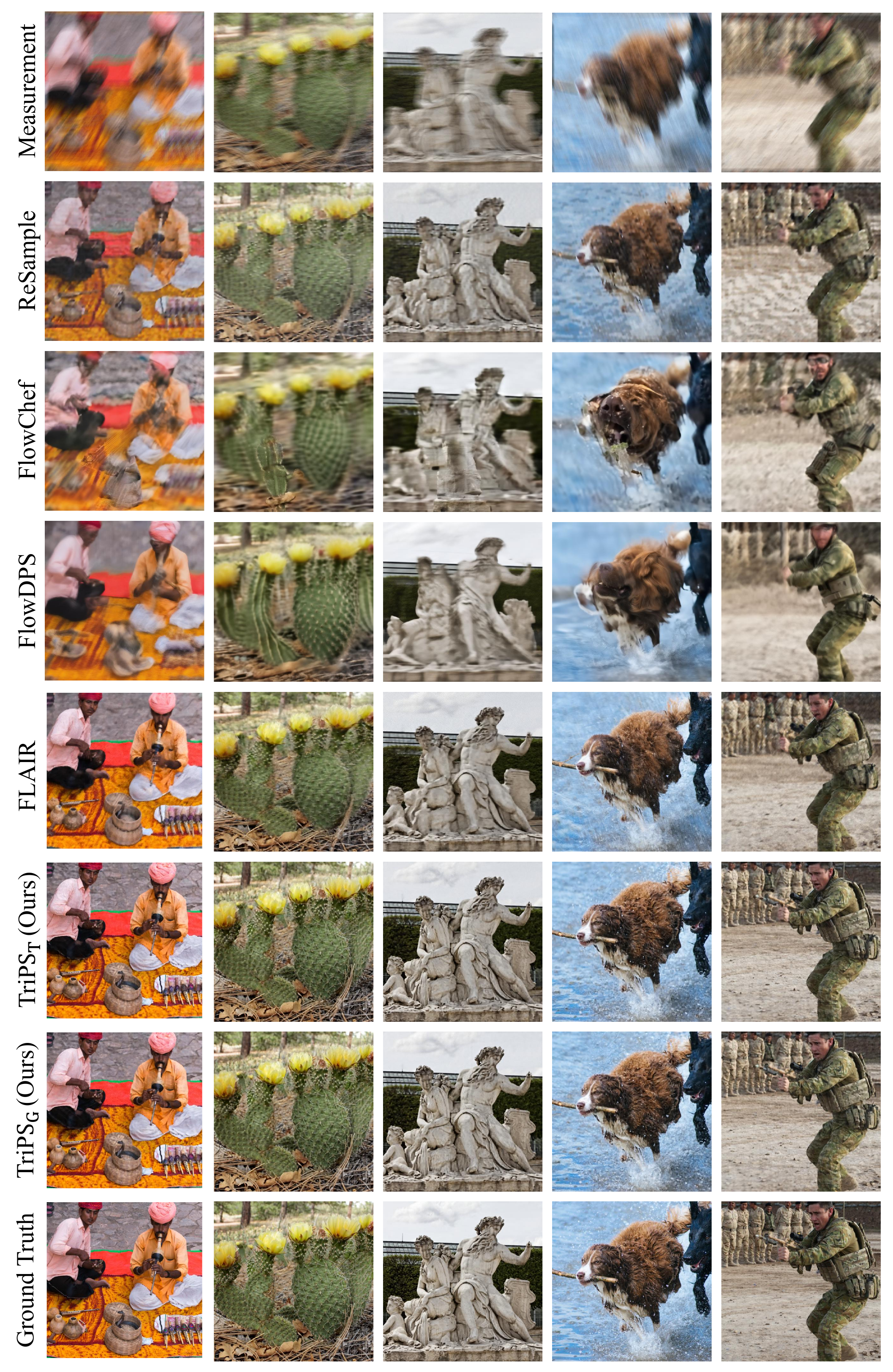}
  \caption{
  Additional qualitative comparison for the motion deblurring on the DIV2K dataset.
  }
  \label{fig:Add_app_quali_MD_DIV2K}
\end{figure*}

\newpage

\begin{figure*}[tp]
  \centering
  \includegraphics[width=0.83\textwidth]{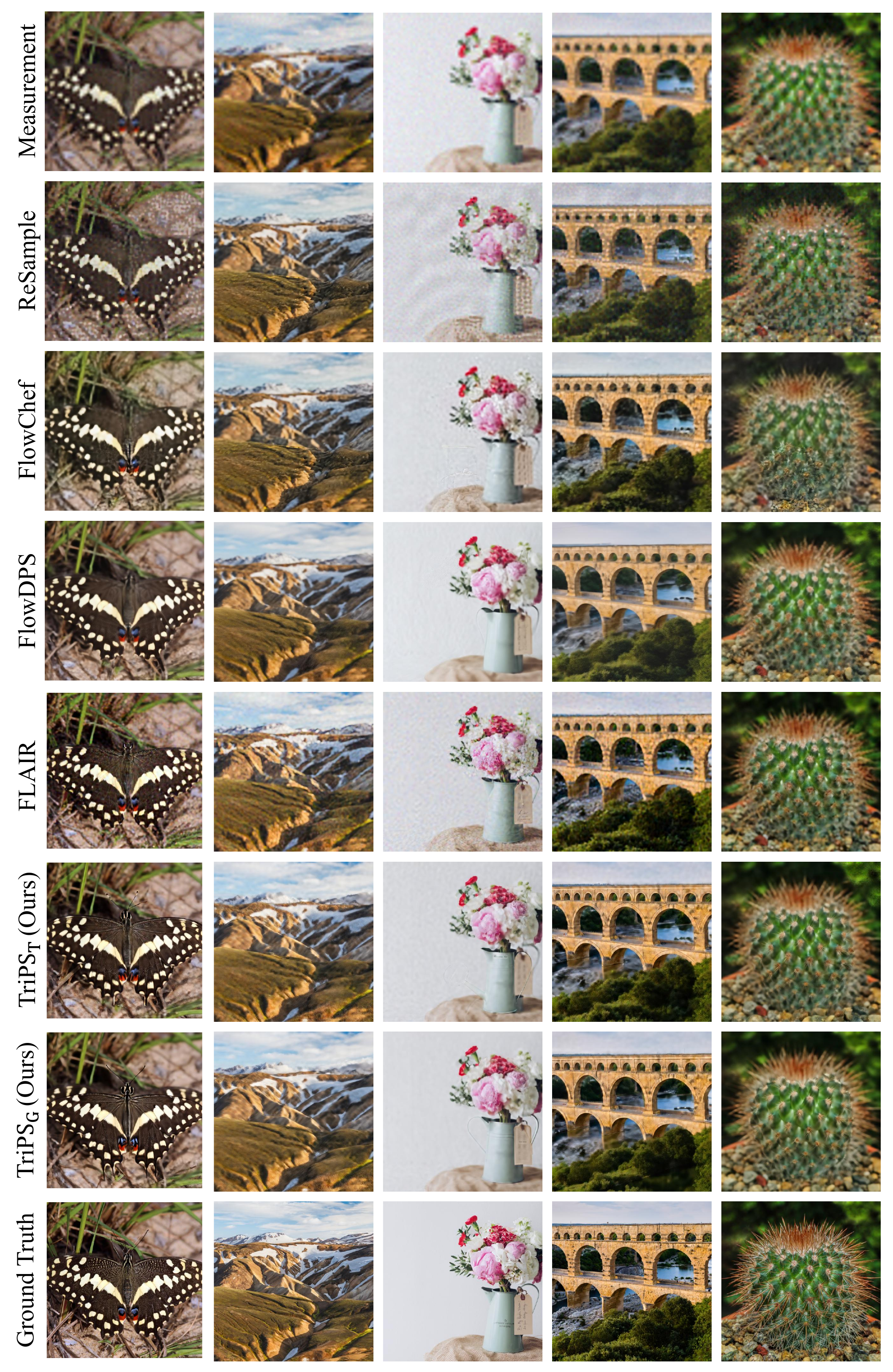}
  \caption{
  Additional qualitative comparison for the super-resolution $\times$12 on the DIV2K dataset.
  }
  \label{fig:Add_app_quali_SRx12_DIV2K}
\end{figure*}

\newpage

\begin{figure*}[tp]
  \centering
  \includegraphics[width=0.83\textwidth]{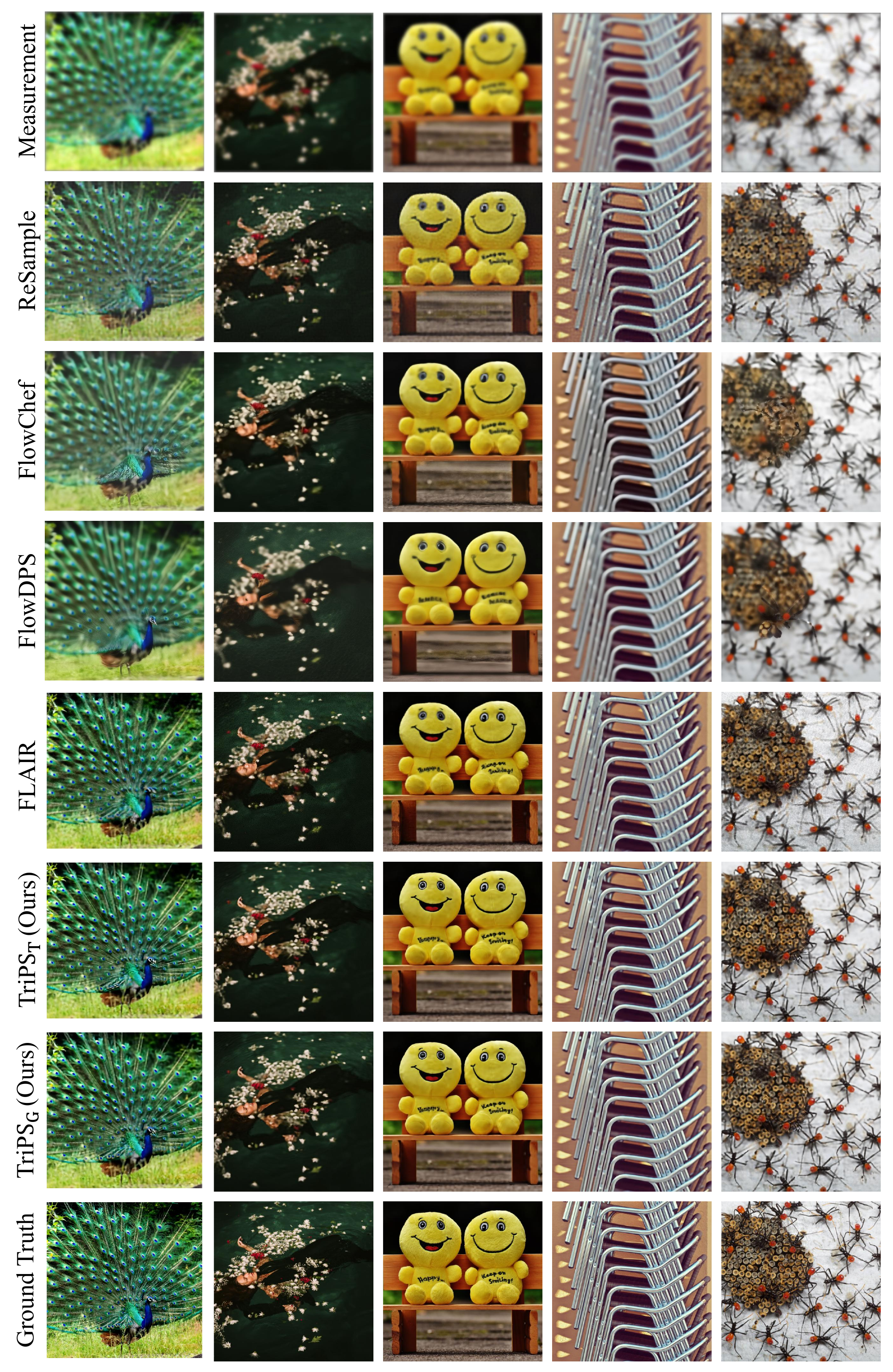}
  \caption{
  Additional qualitative comparison for the gaussian deblurring on the DIV2K dataset.
  }
  \label{fig:Add_app_quali_GD_DIV2K}
\end{figure*}



\end{document}